\begingroup\expandafter\expandafter\expandafter\endgroup
\expandafter\ifx\csname pdfsuppresswarningpagegroup\endcsname\relax
\else
\fi
\documentclass[a4paper,10pt]{article}
\usepackage[utf8]{inputenc}
\usepackage{tikz}
\usepackage{amsfonts,bm}
\usepackage{amsmath}
\usepackage{mathtools}
\usepackage{amssymb}
\PassOptionsToPackage{hyphens}{url}\usepackage[hidelinks]{hyperref}
\usepackage{graphicx}
\usepackage{tcolorbox}
\usepackage[numbers,sort&compress]{natbib}
\DeclareMathOperator*{\argmin}{arg\,min}
\DeclareMathOperator*{\argmax}{arg\,max}
\newcommand{\E}{\mbox{\bf E}}

\usepackage{float,xspace}
\usepackage{cprotect}
\usepackage[capitalise]{cleveref}
\usepackage{tcolorbox}

\usepackage{float,subfig}
\usepackage{forloop}
\newcommand{\defcal}[1]{\expandafter\newcommand\csname c#1\endcsname{{\mathcal{#1}}}}
\newcommand{\defbb}[1]{\expandafter\newcommand\csname b#1\endcsname{{\mathbb{#1}}}}
\newcounter{calBbCounter}
\forLoop{1}{26}{calBbCounter}{
	\edef\letter{\Alph{calBbCounter}}
	\expandafter\defcal\letter
	\expandafter\defbb\letter
}
\usepackage[T1]{fontenc}    
\usepackage{booktabs}       
\usepackage{amsfonts}       
\usepackage{nicefrac}       
\usepackage{microtype}      
\usepackage{algorithm,algorithmic}
\newcommand{\INPUT}{\item[{\bf Input:}]}
\newcommand{\OUTPUT}{\item[{\bf Output:}]}
\newcommand{\Opt}{\text{\upshape{\texttt{OPT}}}\xspace}
\usepackage{amsthm}
\newtheorem{theorem}{Theorem}

\newtheorem{definition}[theorem]{Definition}
\newtheorem{lemma}[theorem]{Lemma}
\newtheorem{corollary}[theorem]{Corollary}

\newcommand{\AlgBarrier}{{\textsc{\textsc{Barrier-Greedy}}}\xspace}
\newcommand{\AlgImproved}{{\textsc{\textsc{Barrier-Greedy++}}}\xspace}
\newcommand{\AlgHeuristic}{{\textsc{\textsc{Barrier-Heuristic}}}\xspace}
\newcommand{\AlgExchange}{{\textsc{\textsc{ExchangeCandidate}}}\xspace}
\newcommand{\Optset}{S^*}

\crefname{algocf}{Algorithm}{Algorithms}
\Crefname{algocf}{Algorithm}{Algorithms}

\renewcommand{\emptyset}{\varnothing}
\newcommand{\printtime}{01/01/2019}

\usepackage[top=1.in, bottom=1in, left=1in, right=1in]{geometry}

\title{Submodular Maximization Through Barrier Functions}
\author
{
	Ashwinkumar Badanidiyuru\thanks{Google, Mountain View, USA. Email: \texttt{ashwinkumarbv@gmail.com}.}
	\and
	Amin Karbasi\thanks{Yale Institute for Network Science, Yale University. Email: \texttt{amin.karbasi@yale.edu}.}
	\and 
	Ehsan Kazemi\thanks{Yale Institute for Network Science, Yale University. Email: \texttt{ehsan.kazemi@yale.edu}.}
	\and
	Jan Vondr\'ak\thanks{Department of Mathematics, Stanford University, USA. Email: \texttt{jvondrak@stanford.edu}.}
}
\date{}

\begin{document}
	
	\pagenumbering{arabic}
	
	\maketitle
	
\begin{abstract}
	In this paper, we introduce a  novel technique for constrained submodular maximization, inspired by barrier functions in continuous optimization. This connection not only improves the running time for constrained submodular maximization but also provides the state of the art guarantee. More precisely, for maximizing a monotone submodular function subject to the combination of a $k$-matchoid and $\ell$-knapsack constraint (for $\ell\leq k$), we propose a potential function that can be approximately minimized. Once we minimize the potential function up to an $\epsilon$ error it is guaranteed that we have found a feasible set with a $2(k+1+\epsilon)$-approximation factor which can indeed be further improved to $(k+1+\epsilon)$ by an enumeration technique. We extensively evaluate the performance of our proposed algorithm over several real-world applications, including a movie recommendation system, summarization tasks for YouTube videos, Twitter feeds and Yelp business locations, and a set cover problem.
\end{abstract}
\section{Introduction} \label{sec:intro}

In the constrained continuous optimization, barrier functions are usually used to impose an increasingly large cost on a feasible point as it  approaches the boundary of the feasible region
\cite{nocedal2006numerical}.
In effect, barrier functions replace constraints by a penalizing term in the primal objective function so that the solution stays away from the boundary of the feasible region. This is an attempt to approximate a constrained optimization problem with an unconstrained one and to later apply standard optimization techniques. 
While the benefits of barrier functions are studied extensively in the continuous domain \cite{nocedal2006numerical}, their use  in discrete optimization is  not very well understood.

In this paper, we show how discrete barrier functions  manifest themselves in constrained submodular maximization. Submodular functions formalize the intuitive diminishing returns condition, a property that not only allows optimization tractability but also appears in many machine learning applications, including video, image, and text summarization \cite{gygli2015video,tschiatschek2014learning,LB11,mitrovic2018data,feldman2018do},
active set selection in non-parametric learning \cite{MKSK16}, sequential decision making \cite{mitrovic2018submodularity,mitrovic2019adaptive}
sensor placement, information gathering \cite{guestrin2005near}, privacy and fairness \cite{kazemi2018scalable}. 
Formally, for a ground set $\cN$, a non-negative set function $f: 2^{\cN} \rightarrow \bR_{\geq 0}$ is \textbf{submodular} if for all sets $A \subseteq B \subset \cN$ and every element $e \in  \cN \setminus B$, we have 
\[f(A \cup \{e\}) - f(A) \geq f(B \cup \{e\}) - f(B) \enspace.\]
The submodular function $f$ is monotone if for all $A \subseteq B$ we have $f(A) \leq f(B)$. 

The celebrated results of \citet{NWF78} and \citet{FNW78} show that the vanilla greedy algorithm provides an optimal approximation guarantee for maximizing a monotone submodular function subject to a cardinality constraint. However, the performance of the greedy algorithm degrades as the feasibility constraint becomes more complex. For instance, the greedy algorithm does not provide any constant factor approximation guarantee if we replace the cardinality constraint with a knapsack constraint. Even though there exist many works that achieve the tight approximation guarantee for maximizing a monotone submodular function subject to multiple knapsack constraints, the running time of these algorithms is prohibitive as they either rely on enumerating large sets or running the continuous greedy algorithm. In contrast,  we showcase  a fundamentally new optimization technique through a discrete barrier function minimization in order  to efficiently handle knapsack constraints and develop fast algorithms. More formally, we consider the following constrained submodular maximization problem defined over the ground set $\cN$:
\begin{align} \label{eq:problem}
	\Optset = \underset{\substack{S\subseteq \cN,\; S \in \cI \\ c_i(S) \leq 1 \; \forall\; i \in [\ell]}}{\argmax}  \mspace{-10mu} f(S) \enspace,
\end{align}
where the constraint is the intersection of a  $k$-matchoid constraint $\cM(\cN, \cI)$ (a general subclass of $k$-set systems) and $\ell$ knapsacks constraints $c_i$ (for $i\in[\ell]$).

\paragraph{Contributions.} 
We propose two algorithms for maximizing a monotone and submodular function subject to the intersection of a $k$-matchoid and $\ell$ knapsack constraints.
Our approach uses a novel barrier function technique and lies  in between fast thresholding algorithms with suboptimal approximation ratios and slower algorithms that use continuous greedy and rounding methods. 
 The first algorithm, \AlgBarrier, obtains a $2(k+1+\epsilon)$-approximation ratio and runs in $\tilde{O}(n r^2)$ time, where $r$ is the maximum cardinality of a feasible solution. The second algorithm, \AlgImproved, obtains a better approximation ratio of $(k+1+\epsilon)$, but at the cost of $\tilde{O}(n^3 r^2)$ running time.
 Our algorithms are theoretically fast and even exhibit better performance in practice while achieving a near-optimal approximation ratio.
 Indeed, the factor of $k+1$ matches the greedy algorithm for $k$ matroid constraints \cite{FNW78}. The only known improvement of this result requires a more sophisticated (and very slow) local-search algorithm \cite{LSV09}.
Our results show that barrier function minimization techniques provide a versatile algorithmic tool for constrained submodular optimization with strong theoretical guarantees that may scale to many previously intractable problem instances. 
Finally, we demonstrate the effectiveness of our proposed algorithms over several real-world applications, including a movie recommendation system, summarization tasks for YouTube videos, Twitter feeds of news agencies and Yelp business locations, and a set cover problem.

\paragraph{Paper Structure.} In \cref{sec:preliminaries}, we formally define the notation and the constraints we use. In \cref{sec:matroids+knapsacks}, we describe our proposed barrier function. We then present our algorithms for maximizing a monotone submodular function subject to a $k$-matchoid system and $\ell$ knapsack constraints. In \cref{sec:heuristic}, built upon of theoretical results, we present a heuristic algorithm with a better performance in practice.
In \cref{sec:experiments}, we describe the experiments we conducted to study the empirical performance of our algorithms.

\section{Related Work}\label{sec:related}

The problem of maximizing a monotone submodular function subject to various constraints goes back to the seminal work of \citet{NWF78} and \citet{FNW78} which showed that the greedy algorithm gives a $(1-1/e)$-approximation subject to a cardinality constraint, and more generally a $1/p$-approximation for any $p$-system (which subsumes the intersection of $p$ matroids, and also the $p$-matchoid constraint considered here). \citet{NW78} also showed that the factor of $1-1/e$ is best possible in this setting. After three decades, there was a resurgence of interest in this area due to new applications in economics, game theory and machine learning. While we cannot do justice to all the work that has been done in submodular maximization, let us mention the works most relevant to ours---in particular focusing on matroid/matchoid and knapsack constraints.

\citet{Sviri04} gave the first algorithm to achieve a $(1-1/e)$-approximation for submodular maximization subject to a knapsack constraint. This algorithm, while relatively simple, requires enumeration over all triples of elements and hence its running time is rather slow ($\tilde{O}(n^5)$). \citet{Vondrak08} and \citet{CCPV11} gave the first $(1-1/e)$-approximation for submodular maximization subject to a matroid constraint. This algorithm, continuous greedy with pipage rounding, is also relatively slow (at least $\tilde{O}(n^3)$, depending on implementation). Using related techniques, \citet{KST09} gave a $(1-1/e)$-approximation subject to any constant number of knapsack constraints, and \citet{CVZ10} gave a $(1-1/e)$-approximation subject to one matroid and any constant number of knapsack constraint; however, these algorithms are even slower and less practical. 

Following these results (optimal in terms of approximation), applications in machine learning called for more attention being given to running time and practicality of the algorithms (as well as other aspects, such as online/streaming inputs and distributed/parallel implementations, which we do not focus on here). In terms of improved running times, \citet{GRST10} developed fast algorithms for submodular maximization (motivated by the online setting), however with suboptimal approximation factors. \citet{bv2014} provided a $(1-1/e-\epsilon)$-approximation subject to a cardinality constraint using $O(\frac{n}{\epsilon} \log \frac{1}{\epsilon})$ value queries, and subject to a matroid constraint using $O(\frac{n^2}{\epsilon^4} \log^2 \frac{n}{\epsilon})$ queries. Also, they gave a fast thresholding algorithm providing a $1/(p+2\ell+1+\epsilon)$-approximation for a $p$-system combined with $\ell$ knapsack constraints using $O(\frac{n}{\epsilon^2} \log^2 \frac{n}{\epsilon})$ queries. This was further generalized to the non-monotone setting by \citet{mirzasoleiman2016fast}. However, note that in these works the approximation factor deteriorates not only with the $p$-system parameter (which is unavoidable) but also with the number of knapsack constraints $\ell$.

\section{Preliminaries and Notation} \label{sec:preliminaries}

Let $f:2^\cN \to \bR_{\geq 0}$ be a non-negative and monotone submodular function defined over ground set $\cN$.
Given an element $a$ and a set $A$, we use $A + a$ as a shorthand for the union $A \cup \{a\}$. We also denote the marginal gain of adding  $a$ to a $A$ by $f(a \mid A) \triangleq f(A + a) - f(A)$. Similarly, the marginal gain of adding a set $B \subseteq \cN$ to another set $A \subseteq \cN$ is denoted by $f( B \mid A) \triangleq f(B \cup A) - f(A)$.

A \emph{set system} $\cM = (\cN, \cI) \mbox{ with } \cI \subseteq 2^\cN$ is an \emph{independence system} if $\emptyset \in \cI$ and $A \subseteq B$, $B \in \cI$ implies that $A \in \cI$.
In this regard, a set $A \in \cI$ is called independent, and a set $B \notin \cI$ is called dependent. 
A \emph{matroid} is an independence system with the following additional property: if $A$ and $B$ are two independent sets obeying $|A| < |B|$, then there exists an element $a \in B \setminus A$ such that $A + a$ is independent.

In this paper, we consider two different constraints. The first constraint is in an intersection of $k$ matroids or a $k$-matchoid (as a generalization of the intersection of $k$-matroids). The second constraint is the set of $\ell$ knapsacks for $\ell \leq k$. Next, we formally define these constraints.

\begin{definition} \label{def:intersection-matroids}
Let $\cM_1 = (\cN,\cI_1),\ldots,\cM_k=(\cN,\cI_k)$ be $k$ arbitrary matroids over the common ground set $\cN$.
An intersection of $k$ matroids is an independent system  $\cM(\cN, \cI)$ such that $\cI = \{ S \subseteq \cN \mid  \forall i,  S \in \cI_i \}$.
\end{definition}

\begin{definition} \label{def:k-matchoid}
	An independence set system $(\cN, \cI)$ is a $k$-matchoid if there exist $m$ different matroids 
	$(\cN_1, \cI_1), \dotsc, (\cN_m, \cI_m)$ such that $\cN= \cup_{i=1}^m \cN_i$, each 
	element $e \in \cN$ appears in no more than $k$ ground sets among $\cN_1, \dotsc,  \cN_m$ and 
	$\cI = \{ S \subseteq \cN \mid \forall i, \cN_i \cap S \in \cI_i \}$.
\end{definition}

A knapsack constraint is defined by a cost vector $c$ for the ground set $\cN$, where for the cost of a set $S \subseteq \cN$ we have $c(S) = \sum_{e \in S} c_e$. Given a knapsack capacity (or budget) $C$, a set $S \subseteq \cN$ is said to satisfy the knapsack constraint $c$ if $c(S) \leq C$. We assume, without loss of generality, the capacity of all knapsacks are normalized to $1$.

Assume there is a global ordering of elements $\cN = \{1,2,3,\ldots,|\cN|\}$.
For a set $S \subseteq \cN$ and an element $a \in \cN$,  the contribution of $a$ to $S$ (denoted by $w_a$) is the marginal gain of adding element $a$ to all elements of $S$ that are smaller than $a$, i.e., $w_a = f(S \cap [a]) - f(S \cap [a-1])$.
 From the submodularity of $f$, it is straightforward to show that $f(S) = \sum_{a \in S} w_a$.
The benefit of adding $b \notin S $ to set $S$ (denoted by $w_b$)
 is the marginal gain of adding element $b$ to set $S$, i.e., $w_b = f(S+b) - f(S)$.
Furthermore, for each element $a$, $\gamma_a = \sum_{i=1}^{k} c_{i,a}$ represents the aggregate cost of $a$ over all knapsacks.
It is easy to see that $\sum_{i=1}^{k} c_i(S) = \sum_{a \in S} \gamma_a$.
We also denote the latter quantity, the aggregate cost of all elements of $S$ over all knapsack, by $\gamma(S)$. 
Since we have $\ell$ knapsacks and the capacity of each knapsack is normalized to $1$, for any feasible solution $S$, we have always $\gamma(S) \leq k$.




\section{The Barrier Function and Our Algorithms}
\label{sec:matroids+knapsacks}
In this section, we first explain our proposed barrier function. We then present \AlgBarrier and \AlgImproved and prove that these two algorithms, by efficiently finding a local minimum of the barrier function, can efficiently maximize a monotone submodular function subject to the intersection of $k$-matroids and $\ell$ knapsacks. At the end of this section, we demonstrate how our algorithms could be extended to the case of $k$-matchoid constraints. 

\subsection{The \AlgBarrier Algorithm} \label{sec:alg-barrier}
Existing local search algorithms under $k$ matroid constraints try to maximize the objective function over a space of $O(n^k)$ feasible swaps \cite{LMNS09,LSV09}; however, our proposed method, a new {\bf local-search} algorithm called \AlgBarrier, avoids the exponential dependence on $k$ while it incorporates the additional knapsack constraints.
Note that the knapsack constraints generally make the structure of feasible swaps even more complicated.

As a first technical contribution, instead of making the space of feasible swaps huge and more complicated, we incorporate the knapsack constraints into a \textbf{potential function} similar to {\bf barrier functions} in the continuous optimization domain. 
For a set function $f(S)$ and intersection of $k$ matroids and $\ell$ knapsack constraints $c_i(S) \leq 1, i \in [\ell]$, we propose the following potential function:
\begin{align} \label{eq:potential}
\phi(S) = \frac{\Opt - (k+1) \cdot f(S)}{1 - \sum_{i=1}^{\ell} c_i(S)}  \enspace,
\end{align}
where $\Opt$ is the optimum value for Problem \eqref{eq:problem}.
This potential function incorporates the knapsack constraints in a very conservative way: while $\gamma(S)$ for a feasible set $S$ could be as large as $\ell$, we consider only sets with $\gamma(S) \leq 1$, whereas for sets with a larger weight the potential function becomes negative.\footnote{In \cref{sec:heuristic}, we propose a version of our algorithm that is more aggressive towards approaching the boundaries of knapsack constraints.}
We point out that the choice of our potential function works best for a combination of $k$ matroids and $k$ knapsacks. When the number of matroid and knapsack constraints is not equal, we can always add redundant constraints so that $k$ is the maximum of the two numbers.
For this reason, in the rest of this paper, we assume $\ell = k$.

In \AlgBarrier, our main goal is to efficiently {\bf minimize} the potential function in several consecutive sequential rounds.
This potential function is designed in a way such that either the current solution respects all the knapsack constraints or if the solution violates any of the knapsack constraints, we can guarantee that the objective value is already sufficiently large.
Note that the potential function involves the knowledge of $\Opt$---we replace this by an estimate that we can ``guess" (enumerate over) efficiently by a standard technique. 

As a second technical contribution, we optimize the local search procedure for $k$ matroids. More precisely, we improve the previously known $O(n^k)$ running time of \citet{LMNS09}  to a new  method with time complexity of $O(n^2)$. This is accomplished by a novel greedy approach that efficiently searches for the best existing swap, instead of a brute-force search among all possible swaps. With these two points in mind, we now proceed to explain our first proposed algorithm \AlgBarrier, in detail.

In the running of \AlgBarrier, we require an accurate enough estimate of the optimum value $\Opt$ that we denote by $\Omega$. 
Indeed, a technique first proposed by \citet{badanidiyuru2014streaming} can be used to guess such a value:
from the submodularity of $f$, we can deduce that $M \leq \Opt \leq r M$, where $M$ is the largest value in the set $\{f(\{j\}) \mid j \in \cN \}$ and $r$ is the maximum cardinality of a feasible solution.
Then, it suffices to try $O(\frac{\log r}{\epsilon})$ different guesses in the set $\Lambda = \{(1+\epsilon)^{i} \mid \nicefrac{M}{(1+ 
	\epsilon)} \leq (1+\epsilon)^{i} \leq r M \}$  to obtain a close enough estimate of \Opt. 
In the rest of this section, we assume that we have access to a value of $\Omega$ such that $(1-\epsilon) \Opt \leq \Omega \leq \Opt$. 
Using $\Omega$ as an estimate of $\Opt$, our potential function converts to
\[ \phi(S) = \frac{\Omega - (k+1) \cdot f(S)}{1 - \gamma(S)} \enspace . \]
To quantify  the effect of each element $a$  on the potential function $\phi$, as a notion of their individual energy, we define the following quantity:
\begin{align} \label{eq:delta-a}
\delta_a = (k+1)\cdot (1 - \gamma(S)) \cdot w_{a} -(\Omega - (k+1) \cdot f(S)) \cdot \gamma_{a} \enspace. 
\end{align}
The quantity $\delta_a$ measures how desirable an element is with respect to the current solution $S$, i.e., larger values of $\delta_a$ would have a larger effect on the potential function.
Also, any element $a \in S$ with $\delta_a \leq 0$ can be removed from the solution without increasing the potential function (see Lemma~\ref{lem:negative-delta}).

The \AlgBarrier algorithm starts with an empty set $S$ and performs the following steps for at most $r \log (1/\epsilon)$ iterations or till it reaches a solution $S$ such that $f(S) \geq \frac{(1-\epsilon) \Omega}{k+1}$:
Firstly, it finds an element $b \in \cN \setminus S$ with the maximum value of $ \delta_b - \sum_{i \in J_b} \delta_{a_i}$ such that $S-a_i+b \in \cI_i$ for $a_i \in S$ and $i \in J_b \triangleq \{ j \in [k]: S+b \notin \cI_j \}$.
\AlgBarrier computes values of  $\delta_a$ from \cref{eq:delta-a}. Note that, in this step, we need to compute $\delta_a$ for all elements $a \in \cN$ only once and store them; then we can use these pre-computed values to find the best candidate $b$. 
The goal of this step is to find an element $b$ such that its addition to set $S$ and removal of a corresponding set of elements from $S$ decrease the potential function by a large margin while still keeping the solution feasible.
In the second step, \AlgBarrier removes all elements with $\delta_{a} \leq 0$ form set $S$. 
In Lemma~\ref{lem:negative-delta}, we prove that these removals could only decrease the potential function.
The \AlgBarrier algorithm produces a solution with a good objective value mainly for two reasons:
\begin{itemize}
	\item if it continues for $r \log (1/\epsilon)$ iterations, we can prove that the potential function would be very close to $0$, which consequently enables us to guarantee the performance for this case.
	Note that, for our solution, we maintain the invariant that $\gamma(S) < 1$ to make sure the knapsack constraints are also satisfied.
	\item if $f(S) \geq \frac{(1-\epsilon)\Omega}{k+1} $, we would prove that the objective value of one the two feasible sets $S \setminus \{b\}$ and $\{b\}$ is at least  $\frac{(1-\epsilon)\Omega}{2(k+1)}$, where $b$ is the last added element to $S$.
\end{itemize}
The details of \AlgBarrier are described in Algorithm~\ref{alg:matroids+knapsacks}. 
Theorem~\ref{thm:matroids+knapsacks} guarantees the performance of \AlgBarrier.

\begin{algorithm}[htb!]
	\caption{\AlgBarrier}
	\label{alg:matroids+knapsacks}
	\begin{algorithmic}[1]
		\INPUT $f:2^\cN \to \bR_{\geq 0}$, membership oracles for $k$ matroids $\cM_1 = (\cN,\cI_1),\ldots,\cM_k=(\cN,\cI_k)$, and $\ell$ knapsack-cost functions $c_i: \cN \rightarrow [0,1]$.
		\OUTPUT A set $S \subseteq \cN  $ satisfying $S \in \bigcap_{i=1}^{k} \cI_i$ and $c_i(S) \leq 1 \  \forall i$.
		\STATE $M \leftarrow \max_{j \in \cN} f(\{j\})$
		\STATE $\Lambda \gets \{(1+\epsilon)^{i} \mid \nicefrac{M}{(1+ \epsilon)} \leq (1+\epsilon)^{i} \leq r M \}$ as potential estimates  of $\Opt$	
		\FOR{$\Omega \in \Lambda$}
		\STATE $S \gets \emptyset$
		\WHILE {$f(S) < \frac{1-\epsilon}{k+1} \Omega$ and iteration number is smaller than $r \log \frac{1}{\epsilon}$}
		\STATE Find $b \in \cN \setminus S$ and $a_i \in S$ for $i \in J_b = \{ j \in [k]: S+b \notin \cI_j \}$ s.t. $S-a_i+b \in \cI_i$ and
		$ \delta_b - \sum_{i \in J_b} \delta_{a_i}$ is maximized (compute $\delta_a$ from \cref{eq:delta-a})
		\STATE $S \leftarrow S \setminus \{a_i:  i \in J_b \} + b$
	\STATE\textbf{while} {$\exists a \in S$ s.t. $\delta_a \leq 0$} \textbf{do} Remove $a$ from $S$
		\ENDWHILE
		\IF{set $S$ satisfies all the knapsack constraints}
		{
		\STATE $S_\Omega \gets S$ \label{line:return-S}
	}
\ELSE{
	\STATE $S_\Omega \gets \argmax\{f(\{b\}), f(S-b)\}$, where $b$ is the last added element to $S$ \label{line:return-S-b}
}	
		\ENDIF
		\ENDFOR
		\RETURN $\argmax_{\Omega \in \Lambda} f(S_\Omega)$
	\end{algorithmic}
\end{algorithm}

\begin{theorem}
	\label{thm:matroids+knapsacks}
	\AlgBarrier (Algorithm~\ref{alg:matroids+knapsacks}) provides a $2(k+1+\epsilon)$-approximation for the problem of maximizing a monotone submodular function subject to the intersection of $k$ matroids and $\ell$ knapsack constraints (for $\ell \leq k$).
	It also runs in time $O(\frac{n r^2}{\epsilon} \log r \log \frac{1}{\epsilon} )$, where $r$ is the maximum cardinality of a feasible solution.
\end{theorem}

\begin{proof}
We first prove that removing elements $a \in S$ with $\delta_a \leq 0$ could only decrease the potential function $\phi(S)$.

\begin{lemma} \label{lem:negative-delta}
Suppose that $S$ is a current solution such that $\gamma(S) < 1$ and $a \in S$ is such that $\delta_a \leq 0$.
Then if we define $S' = S-a$, we obtain a solution $S'$ such that $\gamma(S') < 1$ and $\phi(S') \leq \phi(S)$.
\end{lemma}

\begin{proof}
	First note that by removing an element, the total cost of knapsacks can only decrease, so we still have $\gamma(S') < 1$, as cost of elements is non-negative in all knapsacks.
	Consider the change in the potential function:
	\begin{align} \label{eq:negative-delta}
	\phi(S') - \phi(S) \nonumber 
	& = \frac{\Omega - (k+1) \cdot f(S')}{1 - \gamma(S')} - \frac{\Omega - (k+1)\cdot f(S)}{1 - \gamma(S)} & \textup{(From \cref{eq:potential})}  &\nonumber \\
	& = \dfrac{\left( (\Omega - (k+1) \cdot f(S')) \cdot (1 - \gamma(S))-(\Omega - (k+1) \cdot f(S)) \cdot (1 - \gamma(S')) \right)} {(1-\gamma(S)) \cdot (1-\gamma(S'))} &
	\end{align}
	By submodularity of function $f$, we have $f(S') = f(S-a) \geq f(S) - w_a$, as for $a \in S$, we have $w_a = f(S \cap [a]) - f(S \cap [a-1])$.
	Also, from the linearity of knapsack costs, we have $\gamma(S') = \gamma(S-a) = \gamma(S) - \gamma_a$.
	Therefore, by applying $f(S') \geq f(S) - w_a$ and $\gamma(S') = \gamma(S) - \gamma_a$ to the right side of \cref{eq:negative-delta}, we get:
	\begin{align*}
	\phi(S') - \phi(S) 
	& \leq \dfrac{\left((\Omega - (k+1) \cdot f(S) + (k+1) \cdot  w_a) \cdot (1 - \gamma(S))-(\Omega - (k+1) \cdot f(S)) \cdot  (1 - \gamma(S) + \gamma_a) \right)}{(1-\gamma(S)) \cdot (1-\gamma(S'))} & \\
	& =  \frac{( (k+1) \cdot w_a \cdot (1 - \gamma(S)) - (\Omega - (k+1) \cdot f(S)) \gamma_a )}{(1-\gamma(S)) \cdot (1-\gamma(S'))} & \\
	& =  \frac{ \delta_a}{(1-\gamma(S)) \cdot (1-\gamma(S'))} \leq 0 \enspace. &\mspace{-2000mu} (\delta_{a} \leq 0  \textrm{ and } \gamma(S) \leq 1)
	\end{align*}
\end{proof}

After removing all elements $a \in S$ with $\delta_a \leq 0$, we obtain a new solution $S$ such that $\delta_a > 0$ for all $a \in S$. 
In the next step, we require to include a new element in order to decrease the potential function the most. The following lemma provides an algorithmic procedure to achieve this goal. Recall that we denote the $i$-th matroid constraint by $\cM_i = (N, \cI_i)$.

\begin{lemma} \label{lem:b-swap}
Assume $\Opt = f(\Optset) \geq \Omega$, and $S$ is the current solution such that $S \in \cap_{i=1}^{k} \cI_i$, $f(S) < \frac{1}{k+1} \Omega$, and $\gamma(S) < 1$.
Assume that for each $a \in S$, $\delta_a > 0$. Given $b \notin S$, let $J_b = \{ i \in [k]: S+b \notin \cI_i \}$,
and $a_i(b) = \argmin \{ \delta_a: a \in S \textrm{ and } S-a+b \in \cI_i \}$ for each $i \in J_b$.
Then there is $b \notin S$ such that
\[ \delta_b - \sum_{i \in J_b} \delta_{a_i(b)} \geq \frac{1}{|\Optset|} \cdot  (1 - \gamma(S))\cdot  (\Omega - (k+1)\cdot  f(S)) \enspace . \]
\end{lemma}

\begin{proof}
	To prove this lemma, we first state the following well-known result for exchange properties of matroids.
	\begin{lemma}[\citep{schrijver2003combinatorial}, Corollary 39.12a] \label{lem:matchingBases}
		Let $\cM=(\cN, \cI)$ be a matroid and let $S, T \in \cI$ with $|S| = |T|$. Then there is a perfect matching  $\pi$ between $S\setminus T$ and $T\setminus S$ such that for every $e \in S \setminus T$, the set $(S \setminus \{e\}) \cup \{ \pi(e) \} $ is an independent set.
	\end{lemma}
	
	Let $\Optset$ be an optimal solution with $\Opt = f(\Optset) \geq \Omega$. Let us assume that $\widetilde{S}_i, \widetilde{\Optset}_i$ are bases of $\cM_i$ containing $S$ and $\Optset$, respectively. By Lemma~\ref{lem:matchingBases}, there is a perfect matching $\Pi_i$ between $\tilde{S}_i \setminus \widetilde{\Optset}_i$ and $\widetilde{\Optset}_i \setminus \widetilde{S}_i$ such that for any $e \in \Pi_i$, $\widetilde{S}_i \Delta e \in \cI_i$.
	For each $b \in \Optset$ and $i \in J_b$ (defined as above, $J_b$ denotes the matroids in which we cannot add $b$ without removing something from $S$), let $\pi_i(b)$ denote the endpoint in $S$ of the edge matching $b$ in $\Pi_i$. This means that $S-\pi_i(b)+b \in \cI_i$.
	
	Since for each $i \in J_b$, we pick $a_i(b)$ to be an element of $S$ minimizing $\delta_a$ subject to the condition $S-a+b \in \cI_i$, and $\pi_i(b)$ is a possible candidate for $a_i$, we have $\delta_{a_i(b)} \leq \delta_{\pi_i(b)}$. Consequently, it is sufficient to bound $\delta_b - \sum_{i \in J_b} \delta_{\pi_i(b)}$ to prove the lemma.
	
	Since each $a \in S$ is matched exactly once in each matching $\Pi_i$, we obtain that each $a \in S$ appears as $\pi_i(b)$ at most $k$ times for different $i \in [k]$ and $b \in \Optset$. 
	Note that it could appear less than $k$ times due to the fact that it might be matched to elements in $\widetilde{\Optset}_i \setminus \Optset$.
	Let us define $T_b$ for each $b \in \Optset$ to contain $\{ \pi_i(b): i \in J_b \}$ plus some arbitrary additional elements of $S$, so that each element of $S$ appears in {\em exactly} $k$ sets $T_b$. Since $\delta_a > 0$ for all $a \in S$, we have
	\[ \delta_b - \sum_{a \in T_b} \delta_a \leq \delta_b - \sum_{i \in J_b} \delta_{\pi_i(b)} \leq \delta_b - \sum_{i \in J_b} \delta_{a_i(b)} \enspace.\]
	Hence it is sufficient to prove that $\delta_b - \sum_{a \in T_b} \delta_a \geq \frac{1}{|\Optset|} (1 - \gamma(S)) (\Omega - (k+1) f(S))$ for some $b \in \Optset$.
	Let us choose a random $b \in \Optset$ and compute the expectation $\E[\delta_b - \sum_{a \in T_b} \delta_a]$. First, since each element of $\Optset$ is chosen with probability $\frac{1}{|\Optset|}$, we obtain
	\begin{align*}
	\E[w_b] = \frac{\sum_{b \in \Optset} w_b}{|\Optset|}  = \frac{\sum_{b \in \Optset} f_S(b)}{|\Optset|} 
	\geq \frac{f_S(\Optset)}{|\Optset|} \geq \frac{(\Omega - f(S))}{|\Optset|} \enspace,
	\end{align*}
	by submodularity. Similarly, since $\Optset$ is a feasible solution, we have
	\[ \E[\gamma_b]= \frac{1}{|\Optset|} \sum_{b \in \Optset} \gamma_b \leq \frac{k}{|\Optset|} \enspace .\]
	Concerning the contribution of the items in $T_b$, we obtain,
	\begin{align*}
	\E[\sum_{a \in T_b} w_a] ={} \frac{1}{|\Optset|} \sum_{b \in \Optset} \sum_{a \in T_b} w_{a} 
	={}  \frac{k}{|\Optset|} \sum_{a \in S} w_a
	= \frac{k}{|\Optset|} \cdot f(S) \enspace,    
	\end{align*}
	using the fact that each $a \in S$ appears in exactly $k$ sets $T_b$. Similarly,
	\[ \E[\sum_{a \in T_b} \gamma_{a}]  = \frac{1}{|\Optset|} \sum_{b \in \Optset} \sum_{a \in T_b} \gamma_{a} = \frac{k}{|\Optset|} \cdot \gamma(S) \enspace . \]
	All together, we obtain
	\begin{align*}
	\E[\delta_b - \sum_{a \in T_b} \delta_a] 
	& =  \E \left[(k+1)\cdot (1 - \gamma(S)) \cdot(w_b - \sum_{a \in T_b} w_{a}) 
	- (\Omega - (k+1)\cdot f(S)) \cdot (\gamma_b - \sum_{a \in T_b} \gamma_{a_i}) \right] \\
	& \geq  \frac{k+1}{|\Optset|} \cdot \left(1 - \gamma(S)) \cdot(\Omega - f(S) - k \cdot f(S)\right) - \frac{1}{|\Optset|} \cdot \left(\Omega - (k+1) \cdot f(S)) \cdot (k - k \cdot \gamma(S) \right)  \\
	& = \frac{1}{|\Optset|} \cdot (1 - \gamma(S)) \cdot (\Omega - (k+1) \cdot f(S)) \enspace.
	\end{align*}
	Since the expectation is at least $\frac{1}{|\Optset|} \cdot (1 - \gamma(S))\cdot (\Omega - (k+1) \cdot f(S))$, there must exist an element $b \in \Optset$ for which the expression is at least the same amount, which proves the lemma.
\end{proof}

Now, we bound the maximum required number of iterations to converge to a solution whose value is sufficiently high.
Let $r = |\Optset|$ and  $\Opt = f(\Optset)$ for the optimal solution $\Optset$.
In Algorithm~\ref{alg:matroids+knapsacks}, we start from $S = \emptyset$ and repeat the following: As long as $\delta_a < 0$ for some $a \in S$, we remove $a$ from $S$.
If there is no such $a \in S$, we find $b \notin S$ such that $ \delta_b - \sum_{i \in J_b} \delta_{a_i(b)} \geq \frac{1}{|\Optset|} (1 - \gamma(S)) (\Omega - (k+1) f(S))$ (see Lemma~\ref{lem:b-swap}); we include element $b$ in $S$ and remove set $J_b$ from $S$.

\begin{lemma} \label{lem:ls-iter}
\AlgBarrier, after at most $ r \log (1/\epsilon)$ iterations, returns a set $S$ such that $f(S) > \frac{1-\epsilon}{k+1} \Omega$. Furthermore, at least one of the two sets $S$ or $S-b$ is feasible, where $b$ is the last element added to $S$.
\end{lemma}

\begin{proof}
	At the beginning of the process, we have $\phi(\emptyset) = \Omega$. Our goal is to show that $\phi(S)$ decreases sufficiently fast, while we keep the invariant $0 \leq \gamma(S) < 1$. 
	
	We know that, from the result of Lemma~\ref{lem:negative-delta}, removing elements $a \in S$ with $\delta_a \leq 0$ can only decrease the value of $\phi(S)$.
	We ignore the possible gain from these steps. When we include a new element $b$ and remove $\{ a_i(b): i \in J_b \}$ from $S$, we get from Lemma~\ref{lem:b-swap}:
	\[ \delta_b - \sum_{i \in J_b} \delta_{a_i(b)} \geq \frac{1}{|\Optset|} \cdot (1 - \gamma(S)) \cdot (\Omega - (k+1) \cdot f(S)) \enspace .\]
	Next, let us relate this to the change in $\phi(S)$. We denote the modified set by $S' = (S+b) \setminus \{ a_i(b): i \in J_b \}$.
	First, by submodularity and the definition of $w_a$, we know that
	\[ f(S') \geq f(S) + w_b - \sum_{i \in J_b} w_{a_i(b)} \enspace .\]
	We also have
	\[ \gamma(S') = \gamma(S) + \gamma_b - \sum_{i \in J_b} \gamma_{a_i(b)} \enspace . \]
	
	First, let us consider what happens when $\gamma(S') \geq 1$. This means that $\gamma_b - \sum_{i \in J_b} \gamma_{a_i(b)} \geq 1 - \gamma(S)$. Since we know that $\delta_b - \sum_{i \in J_b} \delta_{a_i(b)} \geq 0$, this means (by the definitions of $\delta_b$ and $\delta_{a_i(b)}$) that
	\[(k+1)\cdot (w_b - \sum_{i \in J_b} w_{a_i(b)}) \geq \Omega - (k+1) \cdot f(S) \enspace .\]
	In other words, $f(S') \geq f(S) + w_b - \sum_{i \in J_b} w_{a_i} \geq \frac{1}{k+1} \Omega$. Note that $S'$ might be infeasible,
	but $S'-b$ is feasible (since $S$ was feasible), so in this case we are done.
	
	In the following, we assume that $\gamma(S') < 1$. Then the potential change is
	\begin{align*}
	\phi(S')  - \phi(S)  &
	 \leq \Bigg( \dfrac{\left (\Omega - (k+1) \cdot (f(S) + w_b - \sum_{i \in J_b} w_{a_i(b)})\right) \cdot (1 - \gamma(S)) }{(1-\gamma(S))\cdot(1-\gamma(S'))} \\
	 & \mspace{150mu}- \dfrac{  (\Omega - (k+1) \cdot  f(S))  \cdot (1 - \gamma(S) - \gamma_b + \sum_{i \in J_b} \gamma_{a_i(b)} )) }{(1-\gamma(S))\cdot(1-\gamma(S'))}  \Bigg)\\
	& = \dfrac{\left((k+1) \cdot (-w_b + \sum_{i \in J_b} w_{a_i(b)}) \cdot (1 - \gamma(S))-(\Omega - (k+1) \cdot f(S)) \cdot (- \gamma_b + \sum_{i \in J_b} \gamma_{a_i(b)})\right)}{(1-\gamma(S))\cdot(1-\gamma(S'))} \\
	& =  \frac{(-\delta_b + \sum_{i \in J_b} \delta_{a_i(b)})}{(1-\gamma(S))\cdot (1-\gamma(S'))}   
	\leq  -\frac{1}{|\Optset|} \frac{\Omega - (k+1)\cdot f(S)}{1 - \gamma(S')}  \\
	& =  -\frac{1}{r} \frac{1-\gamma(S)}{1-\gamma(S')} \phi(S) \enspace,	\end{align*}
	using Lemma~\ref{lem:b-swap}. We infer that
	\begin{align*}
	\phi(S') \leq  \left( 1 - \frac{1}{r} \cdot \frac{1-\gamma(S)}{1-\gamma(S')} \right) \phi(S) \enspace.
	\end{align*}
	By induction, if we denote by $S_t$ the solution after $t$ iterations,
	\begin{align*}
	\phi(S_t)
	& \leq  \prod_{i=1}^{t} \left( 1 - \frac{1}{r} \cdot \frac{1-\gamma(S_{i-1})}{1-\gamma(S_i)} \right) \phi(S_0) 
	\leq  e^{-\frac{1}{r} \sum_{i=1}^{t}  \frac{1-\gamma(S_{i-1})}{1-\gamma(S_i)}} \phi(S_0) \enspace.
	\end{align*}
	Here, we use the arithmetic-geometric-mean inequality: 
	\begin{align*}
	\frac{1}{t} \sum_{i=1}^{t} \frac{1-\gamma(S_{i-1})}{1-\gamma(S_i)}
	\geq \left( \prod_{i=1}^{t} \frac{1-\gamma(S_{i-1})}{1-\gamma(S_i)} \right)^{1/t} 
	= \left( \frac{1-\gamma(S_0)}{1-\gamma(S_t)} \right)^{1/r} \geq 1 \enspace.
	\end{align*}
	Therefore, we can upper bound the potential function at the iteration $t$:
	\begin{align*}
	\phi(S_t)
	\leq  e^{-\frac{t}{r} \cdot \frac{1}{t} \sum_{i=1}^{t}  \frac{1-\gamma(S_{i-1})}{1-\gamma(S_i)}} \phi(S_0)
	\leq e^{-\frac{t}{r}} \phi(S_0) = e^{-\frac{t}{r}} \Omega \enspace.
	\end{align*}
	For $t = r \log \frac{1}{\epsilon}$, we obtain $\phi(S_t) = \frac{\Omega - (k+1)\cdot  f(S_t)}{1-\gamma(S_t)}  \leq \epsilon \Omega$ (and $0 \leq \gamma(S_t) < 1$), which implies $f(S_t) \geq \frac{1-\epsilon}{k+1} \Omega$.
\end{proof}

 Now, we have all the required material to prove Theorem~\ref{thm:matroids+knapsacks}.

\paragraph{Proof of Theorem~\ref{thm:matroids+knapsacks}}
The for loop for estimating $\Opt$ is repeated $\frac{1}{\epsilon} \log r$ times. Consider the value of $\Omega$ such that $(1-\epsilon) OPT \leq \Omega \leq OPT$. 
We perform the local search procedure: In each iteration, we check all possible candidates $b \in \cN \setminus S$ and find the best swap $a_i$ for each matroid $\cM_i$ where a swap is needed (the set of indices $J_b$). This requires checking the membership oracles for $\cM_i$ and the values $\delta_{a_i}$ for each potential swap. This takes $O(r n)$ steps. Note that assume $k$ to be a constant, but generally, it contributes only to the multiplicative constant rather than the degree of the polynomial. Finally, we choose the elements $b \notin S$ and $a_i \in S$ so that $\delta_b - \sum_{i \in J_b} \delta_{a_i}$ is maximized.  Due to Lemma~\ref{lem:b-swap}, the best swap satisfies $\delta_b - \sum_{i \in J_b} \delta_{a_i} \geq \frac{1}{r} \cdot (1 - \gamma(S)) \cdot (\Omega - (k+1) \cdot f(S))$. Following this swap, we need to recompute the values of $\delta_a$ for $a \in S$ and remove all elements with $\delta \leq 0$. Considering Lemma~\ref{lem:ls-iter}, this is sufficient to prove that we terminate within $O(r \log \frac{1}{\epsilon})$ iterations of the local search procedure.
Therefore, the algorithm terminates within running time $O(\frac{n r^2}{\epsilon} \log r \log \frac{1}{\epsilon})$. In the end, we have a set $S$ such that $f(S) \geq \frac{1-\epsilon}{k+1} \Omega$ (as the result of Lemma~\ref{lem:ls-iter}). It is possible that $S$ is infeasible, but both $S-b$ and $b$ are feasible (where $b$ is the last-added element), and by submodularity one of them has an objective value of at least $\frac{1-\epsilon}{2k+2} \Omega$.
\end{proof}

\subsection{The \AlgImproved Algorithm} \label{sec:AlgImproved}

In this section, we use an enumeration technique to improve the approximation factor of \AlgBarrier to $(k+1+\epsilon)$. 
For this reason, we propose the following modified algorithm: for each feasible pair of elements $\{a', a''\}$, define a reduced instance where the objective function $f$ is replaced by a monotone and submodular function $g(S) \triangleq f(S \cup \{a',a''\}) - f(\{a',a''\})$, and the knapsack capacities are decreased by $c_{i,a'} + c_{i,a''}$. 
In this reduced instance, we remove the two elements $a',a''$ and all elements $a \in \cN \setminus \{a',a'' \}$ with $g(\{ a \}) > \frac12 f(\{a',a''\})$ from the ground set $\cN$. 
Recall that the contraction of a matroid $\cM_i = (\cN_i, \cI_i)$ to a set $A$ is defined by a matroid $\cM'_i= (\cN \setminus A,\cI'_i)$ such that  $\cI'_i = \{ S \subseteq \cN \setminus A : S \cup A \in \cI_i\}$.
In the reduced instance, we consider contractions of all the $k$ matroids to set $\{a',a''\}$ as the new set of matroid constraints. Note that elements $a$ with $g(\{ a \}) > \frac12 f(\{a',a''\})$ are also removed from the ground set of these contracted matroids.
Then, to obtain a solution $S_{a',a''}$, we run Algorithm~\ref{alg:matroids+knapsacks} on the reduced instance. Finally, we return the best solution of $S_{a',a''} \cup \{a',a''\}$ over all feasible pairs $\{a',a''\}$.
Here, by construction, we are sure that all the solutions $S_{a',a''} \cup \{a',a''\}$ are feasible in the original set of constraints.
Note that,  for the final solution, if there is no feasible pair of elements, we just return the most valuable singleton.
The details of our algorithm (called \AlgImproved) are described in Algorithm~\ref{alg:improve}. 
Theorem~\ref{thm:matroids+knapsacks+} guarantees the performance of \AlgBarrier.

\begin{algorithm}[htb!]
	\caption{\AlgImproved}
	\label{alg:improve}
	\begin{algorithmic}[1]
		\INPUT $f:2^ \cN \rightarrow \bR_{\geq 0}$, membership oracles for $k$ matroids $\cM_1 = (\cN,\cI_1),\ldots,\cM_k=(\cN,\cI_k)$, and $\ell$ knapsack-cost functions $c_i: \cN \rightarrow [0,1]$.
		\OUTPUT A set $S \subseteq \cN  $ satisfying $S \in \bigcap_{i=1}^{k} \cI_i$ and $c_i(S) \leq 1 \ \forall i$.
		\FOR{each feasible pair of elements $\{a', a''\}$}
		\STATE $g(S)  \triangleq f(S \cup \{a',a''\}) - f(\{a',a''\}).$
		\STATE  Decrease  the knapsack capacities by $c_{i,a'} + c_{i,a''}.$
		\STATE Let $\cN' \gets \cN \setminus ( \{a',a''\} \cup \{a \mid g(a) > \frac12 f(\{a',a'' \}))$ and contracts all matroid constraints $\cM_i(\cN_i, \cI_i)$ by set $\{a',a''\}$.
		\STATE Run Algorithm~\ref{alg:matroids+knapsacks} on the reduced instance $g: 2^{\cN'} \to \bR_{\geq 0}$, to obtain a solution $S_{a',a''}$.
		\ENDFOR
		\RETURN the best of $S_{a',a''} \cup \{a',a''\}$ over all feasible pairs $\{a',a''\}$ (If there is no feasible pair of elements, just return the most valuable singleton).
	\end{algorithmic}
\end{algorithm}

\begin{theorem}
\label{thm:matroids+knapsacks+}
\AlgImproved (Algorithm~\ref{alg:improve}) provides a $(k+1+\epsilon)$-approximation for the problem of maximizing a monotone submodular function subject to the intersection of $k$ matroids and $\ell$ knapsack constraints (for $\ell \leq k$).
It also runs in time $O(\frac{n^3 r^2}{\epsilon} \log r \log \frac{1}{\epsilon} )$, where $r$ is the maximum cardinality of a feasible solution.
\end{theorem}

\begin{proof}
	Since we enumerate over $O(n^2)$ pairs of elements, the running time is $O(n^2)$ times the running time of Algorithm~\ref{alg:matroids+knapsacks}.
	
	Consider an optimal solution $\Optset$ and a greedy ordering of its elements with respect to $f$. 
	Also, consider the run of the algorithm, when $a',a''$ are the first two elements of $\Optset$ in the greedy ordering. 
	Note that if all optimal solutions have only one element, it means there is no feasible pair, due to the monotonicity of $f$. 
	In this case, we just return the best singleton, which is optimal. 
	All elements of $\Optset$ following $a',a''$ in the greedy ordering have a marginal value of at most $\frac12 f(\{a',a''\})$, by the greedy choice of $a',a''$. 
	Therefore, these elements are still present in the reduced instance.
	Furthermore, since $\Optset \setminus \{a',a''\}$ is a feasible solution in the reduced instance,
	Algorithm~\ref{alg:matroids+knapsacks} always finds a solution: if the produced set $S$ by Algorithm~\ref{alg:matroids+knapsacks} is feasible, then the solution is returned at Line~\ref{line:return-S} of that algorithm with a guarantee:
	\[g(S) \geq \frac{1-\epsilon}{k+1} \cdot g(\Optset \setminus \{a',a''\}) = \frac{1-\epsilon}{k+1} \cdot (OPT - f(\{a',a''\})) \enspace. \]
	However, the set $S$ could be potentially infeasible and the solution then is returned at Line~\ref{line:return-S-b} of Algorithm~\ref{alg:matroids+knapsacks}. 
	In this case, we know that $S-b$ is feasible in the reduced instance where $b$ is the last-added element, and hence $S-b+a'+a''$ is feasible in the original instance. 
	Also, $g(b) \leq \frac12 f(\{a',a''\})$, otherwise $b$ would not be present in the reduced instance. 
	By submodularity, the value of $S-b+a'+a''$ is at least
	\begin{align*}
	f(S-b+a'+a'') & =   f(\{a',a''\}) + g(S-b) 
	\geq  f(\{a',a''\}) + g(S) - g(\{b\})\\
	& \geq f(\{a',a''\}) +\frac{1-\epsilon}{k+1} \cdot (OPT - f(\{a',a''\})) - \frac12 f(\{a',a''\}) \\
	&  \geq  \frac{1-\epsilon}{k+1} \cdot OPT \enspace.
	\end{align*}
	Since $S + a' + a''$ or $S-b+a'+a''$ is one of the considered solutions, we are done.
\end{proof}

\subsection{The Generalization to $k$-matchoids} \label{sec:k-matchoid}

In this section, we show that our algorithms could be extended to $k$-matchoids, a more general class of constraints.
To achieve this goal, we need to slightly modify the \AlgBarrier algorithm in order to make it suitable for the $k$-matchoid constraint.
More specifically, for each element $b \in S$, we use \AlgExchange to find a set $U_b \subseteq S$ such that $(S \setminus U_b) + b$ satisfies the $k$-matchoid constraint where exchanges are done with elements with the minimum values of $\delta_a$.
The pseudocode of \AlgExchange is given as Algorithm~\ref{alg:exchange_alg}.

\begin{algorithm}[htb!]
	\caption{\AlgExchange($S,b$)}
	\label{alg:exchange_alg}
		\begin{algorithmic}[1]
\STATE	Let $U \gets \varnothing$.
	\FOR{$i = 1$ \textbf{to} $m$}
		\IF{$(S + b) \cap \cN_i \not \in \cI_i$}
		\STATE	Let $A_i \gets \{a \in S \mid ((S - a + b) \cap \cN_i) \in \cI_\ell\}$.
		\STATE	Let $a_i \gets \arg \min_{a \in A_\ell} \delta_{a}$.
		\STATE	Add $a_i$ to $U$.
		\ENDIF
	\ENDFOR
	\RETURN $U$.
	\end{algorithmic}
\end{algorithm}

In order to guarantee the performance our proposed algorithms under the $k$-matchoid constraint, we provide the following lemma which is the equivalent of Lemma~\ref{lem:b-swap} for $k$-matchoid.

\begin{lemma} \label{lemma:matchoid}
Assume $\Opt = f(\Optset) \geq \Omega$, and $S$ is the
current solution that satisfies the $k$-matchoid constraint $\cM(\cN, \cI)$ with $f(S) < \frac{1}{k+1} \Omega$, and $\gamma(S) < 1$.
Then there is $b \notin S$ such that
\[ \delta_b - \sum_{i \in J_b} \delta_{a_i(b)} \geq \frac{1}{|\Optset|} (1 - \gamma(S)) (\Omega - (k+1) f(S)) \enspace . \]
\end{lemma}

\begin{proof}
	For the sake of simplicity of the analysis, we assume that every element $a \in \cN$ belongs to exactly $k$ out of the $m$ ground sets $\cN_i$ ($i \in [m]$) of the
	matroids defining $\cN$.
	To make this assumption valid, for every element $a \in \cN$ that belongs to the ground sets of only $k' < k$ out of the $m$ matroids, we add $a$ to $k - k'$ additional matroids as an element whose addition to an independent set always keeps the set independent. 
	It is easy to observe that the addition of $a$ to these matroids does not affect the behavior of our Algorithms.
	
	Let us assume that $\widetilde{S}_i, \widetilde{\Optset}_i$ are bases of $\cM_i$ containing $S \cap \cN_i$ and $\Optset \cap \cN_i$, respectively. 
	By Lemma~\ref{lem:matchingBases}, there is a perfect matching $\Pi_i$ between $\tilde{S}_i \setminus \widetilde{\Optset}_i$ and $\widetilde{\Optset}_i \setminus \widetilde{S}_i$ such that for any $e \in \Pi_i$ we have $\widetilde{S}_i\Delta e \in \cI_i$.
	For each $b \in \Optset$ and $i \in J_b$ where we define $J_b = \{ i \in [m] \mid (S + b) \cap \cN_i \not \in \cI_i \}$, let $\pi_i(b)$ denote the endpoint in $S$ of the edge matching $b$ in $\Pi_i$. This means that $S-\pi_i(b)+b \in \cI_i$.
	Since for each $i \in J_b$, we pick $a_i(b)$ to be an element of $S$ minimizing $\delta_a$ subject to the condition $S-a+b \in \cI_i$, and $\pi_i(b)$ is a possible candidate for $a_i$, we have $\delta_{a_i(b)} \leq \delta_{\pi_i(b)}$. Consequently, it is sufficient to bound $\delta_b - \sum_{i \in J_b} \delta_{\pi_i(b)}$ to prove the lemma.
	Since each $a \in S$ is matched at most once in each matching $\Pi_i$, we obtain that each $a \in S$ appears as $\pi_i(b)$ at most $k$ times for different $i \in [m]$ and $b \in \Optset$. Note that it could appear less than $k$ times.
	We can then define $T_b$ for each $b \in \Optset$ to contain $\{ \pi_i(b): i \in J_b \}$ plus some arbitrary additional elements of $S$, so that each element of $S$ appears in exactly $k$ sets $T_b$. 
	By providing this exchange property for $k$-matchoids, the rest of the proof is exactly the same as proof of Lemma~\ref{lem:b-swap}.
\end{proof}

From the result of Lemma~\ref{lemma:matchoid} and Theorems~\ref{thm:matroids+knapsacks} and \ref{thm:matroids+knapsacks+}, we conclude the following corollaries for maximizing a monotone and submodular function subject to a $k$-matchoid and $\ell$ knapsack constraints. 

\begin{corollary} \label{cor:matchoid-approx}
\AlgBarrier (Algorithm~\ref{alg:matroids+knapsacks}) provides a $2(k+1+\epsilon)$-approximation for the problem of maximizing a monotone submodular function subject to $k$-matchoid and $\ell$ knapsack constraints (for $\ell \leq k$).
\end{corollary}

\begin{corollary} \label{cor:matchoid-approx-improved}
\AlgImproved (Algorithm~\ref{alg:improve}) provides a $(k+1+\epsilon)$-approximation for the problem of maximizing a monotone submodular function subject to $k$-matchoid and $\ell$ knapsack constraints (for $\ell \leq k$).
\end{corollary}

\section{A Heuristic Algorithm} \label{sec:heuristic}

In \cref{sec:matroids+knapsacks}, we proposed \AlgBarrier with the following
interesting property: it needs to consider only sets $S$ where the sum of all the $k$ knapsacks is at most 1 for them, i.e., sets $S$ such that $\gamma(S) = \sum_{i}^{k} \sum_{a \in S}c_{i,a} \leq 1$.
For scenarios with more than one knapsack, while \AlgBarrier theoretically produces a highly competitive objective value, there might be feasible solutions such that they fill the capacity of all  knapsacks, i.e., $\gamma(S)$ could be very close to $k$ for them.
Unfortunately, both our proposed algorithms fail to find these kinds of solutions.
In this section, inspired by our theoretical results, we design a heuristic algorithm (called \AlgHeuristic) that overcomes this issue. More specifically, this algorithm is very similar to \AlgBarrier with two slight modifications: 
(i) Instead of \cref{eq:delta-a}, we use a new formula to calculate the importance of an element $a$ with respect to the potential function:
\begin{align} \label{eq:new-delta-a}
	\delta_a = (k+1) \cdot (\lambda - \gamma(S))  \cdot w_{a} -(\Omega - (k+1) \cdot f(S)) \cdot \gamma_{a} \enspace,
\end{align}
where $1 \leq \lambda \leq k$. This modification allows us to include sets with $\gamma(S) > 1$ for the outcome of algorithms as $\delta_{a}$ could still be non-negative for them.
(ii) The \AlgBarrier is designed in a way such that for a solution $S$, we have $\gamma(S) \leq 1$. 
This fact consequently implies that the set $S$ satisfies all the knapsack constraints; therefore, by the algorithmic design, we can guarantee that knapsacks are not violated.
On the other hand, in \cref{eq:new-delta-a} for values $\lambda > 1$, set $S$ may violate one or more of the knapsack constraints. 
For this reason, we need to choose the element $b$ from a set $\cN'$ such that for all $b \in \cN'$ the set $(S \setminus U_b) + b$ is feasible; and if this set $\cN'$ is empty, i.e., there is no such element $b$, we stop the algorithm and return the solution (see Line~\ref{line:check-knaspack} of Algorithm~\ref{alg:heuristic})).
For the sake of completeness, we provide a detailed description of \AlgHeuristic in Algorithm~\ref{alg:heuristic}.

\begin{algorithm}[htb!]
	\caption{\AlgHeuristic}
	\label{alg:heuristic}
	\begin{algorithmic}[1]
		\INPUT $f:2^ \cN \rightarrow \bR_{\geq 0}$, membership oracles for a $k$-matchoid set system $(\cN, \cI)$, and $\ell$ knapsack-cost functions $c_i: \cN \rightarrow [0,1]$.
		\OUTPUT A set $S \subseteq \cN  $ satisfying $S \in \cI$ and $c_i(S) \leq \forall i$.
		\STATE $M \leftarrow \max_{j \in \cN} f(\{j\})$
		\STATE $\Lambda \gets \{(1+\epsilon)^{i} \mid \nicefrac{M}{(1+ 
			\epsilon)} \leq (1+\epsilon)^{i} \leq r M \}$ as potential estimates  of $\Opt$	
		\FOR{$\Omega  \in \Lambda$}
		\STATE $S \gets \emptyset$.
		\FOR {Iteraton number from $1$ \textbf{to} $r \log \frac{1}{\epsilon}$}
		\STATE $\cN' \gets \{ b \in \cN \setminus S \mid (S \setminus U_b) \cup \{b\} \textrm{ satisfies all } \ell \textrm{ knapsack constraints} \}$, where we have defined $U_b \gets \AlgExchange(S,b)$.
		\STATE \textbf{if} $\cN'=\emptyset$ \textbf{then} \textbf{break}. \label{line:check-knaspack}
		\STATE
		$b \gets \argmax_{b \in \cN'} \left( \delta_b - \sum_{a \in U_b} \delta_{a} \right)$ for $\delta_a = (k+1) \cdot (\lambda - \gamma(S)) \cdot w_{a} - (\Omega - (k+1) \cdot f(S)) \cdot \gamma_{a}$.
		\STATE $S \gets (S \setminus U_b) + b$, where $U_b \gets \AlgExchange(S,b)$.
		\STATE\textbf{while} {$\exists a \in S$ such that $\delta_a \leq 0$} \textbf{do} Remove $a$ from $S$.
		\ENDFOR
		\STATE$S_\Omega \gets S$
		\ENDFOR
		\RETURN $\argmax_{\Omega \in \Lambda} f(S_\Omega)$
	\end{algorithmic}
\end{algorithm}

\section{Experimental Results} \label{sec:experiments}

In this section, we compare the performance of our proposed algorithms with several baselines.
Our first baseline is the vanilla Greedy algorithm. 
It starts with an empty set $ S = \emptyset$ and keeps adding elements one by one greedily (according to their marginal gain) while the $k$-system and $\ell$-knapsack constraints are both satisfied.
Our second baseline, Density Greedy, starts with an empty set $S=\emptyset$ and keeps adding elements greedily by the ratio of their marginal gain to the total knapsack cost of each element (i.e., according to ratio $\nicefrac{f(a \mid S)}{\gamma_a}$ for $e \in \cN$) while the $k$-system and $\ell$-knapsack constraints are satisfied.
We also consider the state-of-the-art algorithm (called Fast) for maximizing monotone and submodular functions under a $k$ matroid constraints and $\ell$ knapsack constraints \cite{bv2014}.
This algorithm is a greedy-like algorithm
with respect to marginal values, while it discards all elements with a density below some threshold. This thresholding idea guarantees that the solution does not exceed the knapsack constraints without
reaching a high enough utility.
The Fast algorithm runs in time $O(\frac{n}{\epsilon^2} \log \frac{n}{\epsilon})$
provides a $(1 + \epsilon)(k + 2\ell + 1)$-approximation.

In \cref{sec:vertex-cover,sec:movie}, we compare the above algorithms on two tasks of vertex cover over real-world networks and video summarization subject to a set system and a single knapsack constraint. Then, in \cref{sec:yelp,sec:twitter,sec:Movielens}, we evaluate the performance of algorithms, respectively, on the Yelp location summarization, Twitter text summarization and movie recommendation applications subject to a set system and multiple knapsack constraints.
Note that the corresponding constraints are explained independently for each specific application.

In our evaluations,  we compare the algorithms based on two criteria: objective value and number of calls to the Oracle.
Our experimental evaluations demonstrate the following facts:
(i) the objective values of the \AlgBarrier algorithm (and also \AlgHeuristic for more than one knapsack) consistently outperform the baseline algorithms,
and (ii) the computational complexities of our proposed algorithms are quite competitive in practice.
Indeed, while the Fast algorithm provides a better computational guarantee, we observe that for several applications our algorithm exhibits a better performance (in terms of the number of calls to the Oracle) than Fast (see \cref{fig:graph-o-1912,fig:graph-o-email,fig:twitter-o-2,fig:yelp-o-k,fig:movie-o-k}).

\subsection{Vertex Cover} \label{sec:vertex-cover}

In this experiment, we compare \AlgBarrier with Greedy, Density Greedy and Fast.
We define a monotone and submodular function over vertices of a directed real-world graph $G = (V, E)$.
Let's $w : V \rightarrow \bR_{\geq 0}$ denotes a weight function on the vertices of graph $G$. 
For a given vertex set $S \subseteq V$, assume $N(S)$ is the set of vertices which are pointed to by $S$, i.e., $N(S) \triangleq \{v \in V \mid \exists u \in S \text{ such that } (u, v) \in E\}$. 
We define $f:2^{V} \to \bR_{\geq 0}$ as follows:
\[
f(S) = \sum_{u \in N(S) \cup S} w_{u}\enspace,
\]
and we assign to each vertex $u$ a weight of one.
In this set of experiments, our objective is to maximize function $f$ subject to the constraint that we have an upper limit $m$ on the total number of vertices we choose, as well as an upper limit $m_i$ on the number of vertices from each social communities. 
For the simplicity of our evaluations, we use a single value for all $m_i$.
This constraint is the intersection of a uniform matroid and a partition matroid.
To assign vertices to different communities, we use the Louvain method \cite{blondel2008fast}.\footnote{Available for download from: \url{https://sourceforge.net/projects/louvain/}}
In addition, for each graph, we reduce the total number of communities to five by merging smaller communities.
For a knapsack constraint $c$, we set the cost of each vertex $u$ as  $c(u) \propto 1 + \max\{0, d(u) - q\}$, where $d(u)$ is the out-degree of node $u$ in graph $G(V,E)$.
We normalize the costs such that the average cost of each element is $\nicefrac{1}{20}$, i.e., 
$\frac{\sum_{u \in V} c(u)}{|V|} = \nicefrac{1}{20}$. 
With this normalization, we expect the average size of the largest set which satisfies the knapsack constraint is roughly close to 20.
In our experiment, we use real-world graphs from \cite{snapnets} and run the algorithms for varying knapsack budgets. 
We also set $m = 15, m_i = 6$ and $q = 6$.

In \cref{fig:graph-cover}, we see the evaluations for two graphs: Facebook ego network and EU Email exchange network. 
From these experiments, it is evident that \AlgBarrier outperforms the other specialized algorithms for this problem in terms of both objective value and computational complexity.
We also observe that the performance of Greedy is slightly worse than Fast.
We should point out that the running times of Greedy and Density Greedy are the two smallest, as these two algorithms do not make any adjustments to make them suitable for the constraints of this application and obviously they do not provide any theoretical guarantees.

\begin{figure*}[htb!] 
	\centering  
	\subfloat[Facebook ego network] {\includegraphics[height=33.mm]{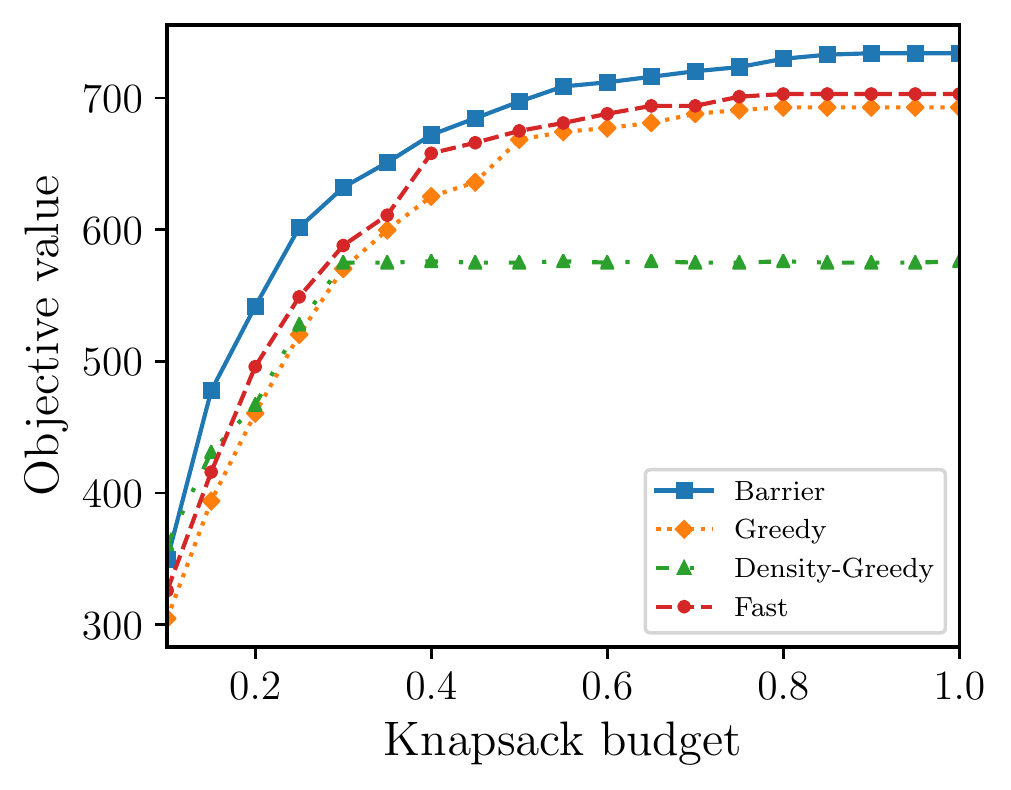}\label{fig:graph-f-1912}}
	\subfloat[EU Email]
	{\includegraphics[height=33.mm]{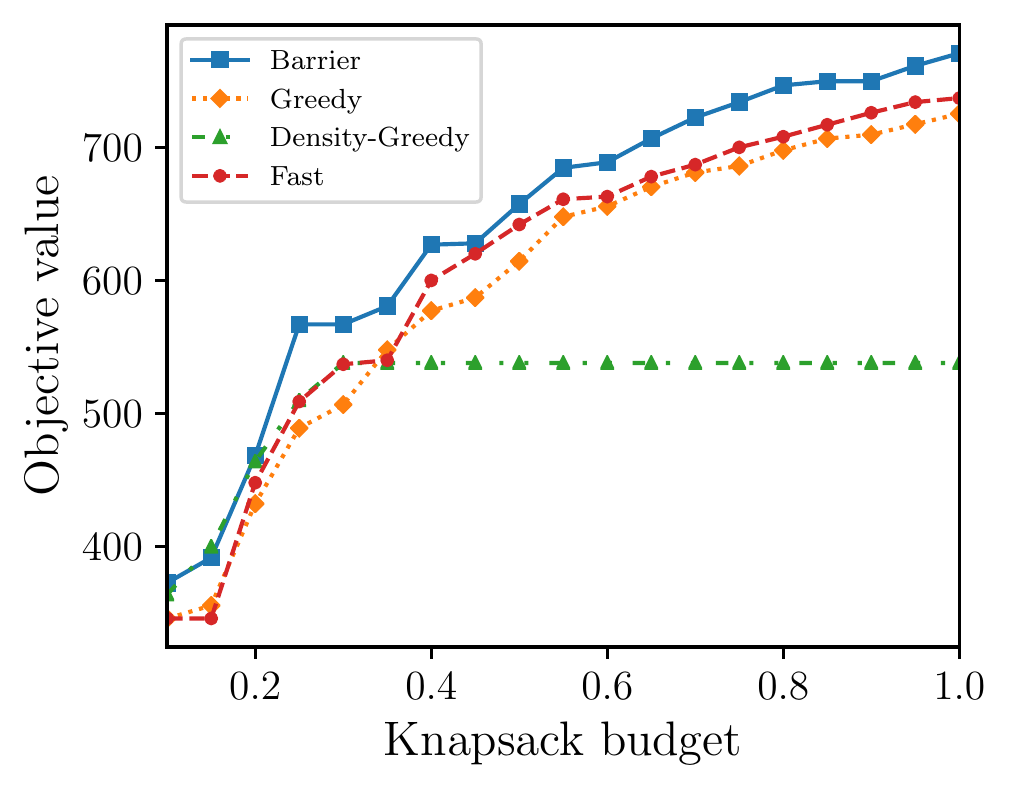}\label{fig:graph-f-email}}
	\subfloat[Facebook ego network]	{\includegraphics[height=33.mm]{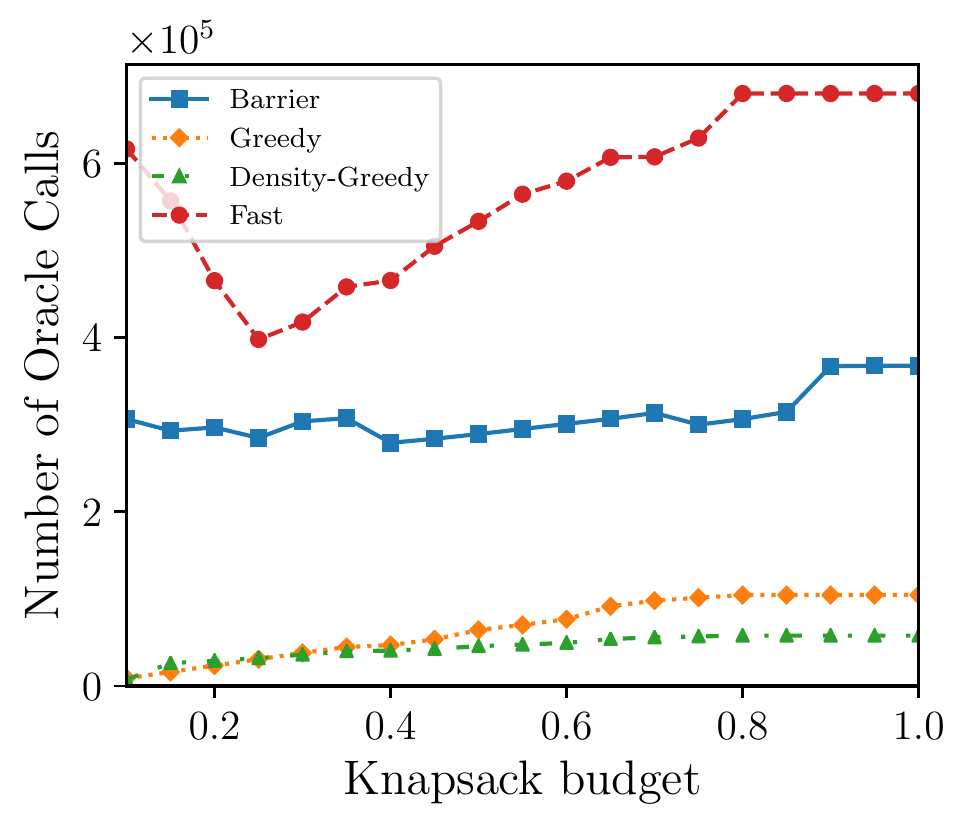}\label{fig:graph-o-1912}}
	\subfloat[EU Email]	{\includegraphics[height=33.mm]{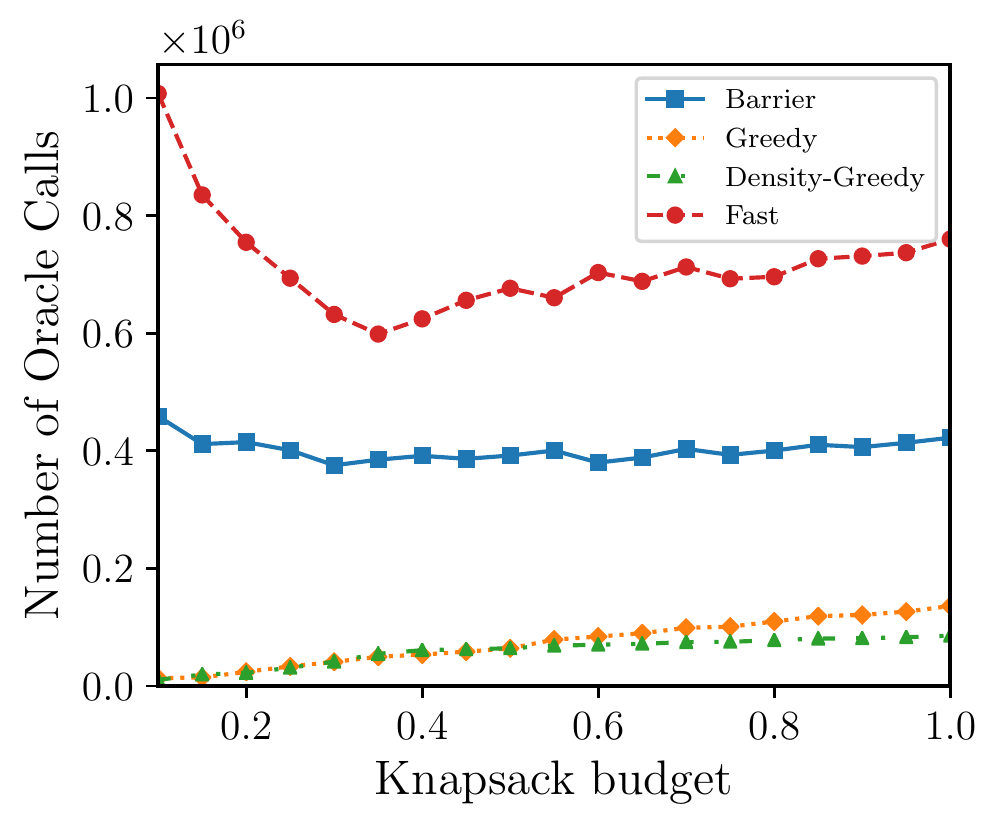}\label{fig:graph-o-email}}
	\label{fig:graph-cover}
	\caption{Vertex cover over real graphs: We compare algorithms for varying knapsack budges based on objective value and number of calls to the Oracle.}
\end{figure*}

\subsection{Video Summarizing Application} 
\label{sec:movie}

Video summarization, as a key step for faster browsing and efficient indexing of large video collections, plays a crucial role in many data mining procedures.
In the second application, we want to summarize a collection of five videos from VSUMM dataset \cite{de2011vsumm}\footnote{Available for download from: \url{https://sites.google.com/site/vsummsite/}}.
Our objective is to select a subset of frames from these videos in order to maximize a utility function $f(S)$ (which represents the diversity of frames). 
 We set limits for the maximum number of allowed frames from each video (referred to as $m_i$), where we consider the same value of $m_i$ for all five videos.
 We also want to bound the total entropy of the selection as a proxy for the storage size of the selected summary.

In order to extract features from frames of each video, we apply a pre-trained ResNet-18 model \cite{he2016deep}.
Then given a set of frames, we define the matrix $M$ such that $M_{ij}= e^{-\lambda \cdot \textrm{dist}(x_i,x_j)}$, where $\text{dist}(x_i,x_j)$ denotes the Euclidean distance between the feature vectors of $i$-th and $j$-th frames, respectively.
Matrix $M$ implicitly represents a similarity matrix among different frames of a video.
The utility of a set $S \subseteq \cN$ is defined as a non-negative and monotone submodular objective $f(S) = \log \det (\mathbf{I} + \alpha M_S)$, where $\mathbf{I}$ is the identity matrix, $\alpha > 0$ and $M_S$ is the principal sub-matrix of similarity matrix $M$ indexed by $S$ \cite{herbrich2003fast}. Informally, this function is meant to measure the diversity of the vectors in $S$.
A knapsack constraint $c$ captures the entropy of each frame. More specifically, for a frame $u$ we define $c(u) = \nicefrac{\mathrm{H}(u)}{20}$.

In \cref{fig:video-f-budget,fig:video-o-budget}, we set the maximum number of allowed frames from each video to $m_i = 10$ and compare the algorithms for varying values of the knapsack budget. We observe that (i) \AlgBarrier returns solutions with a higher utility (up to 50\% more than the second-best algorithm), and (ii) the running time of the Fast algorithm is lower than our proposed algorithm. This experiment showcases the fact that \AlgBarrier effectively trades off some amount of computational complexity in order to increase the objective values by a huge margin.
In \cref{fig:video-f-g,fig:video-o-g}, we evaluate the performance of algorithms based on the maximum number of allowed frames from each video, i.e., $m_i$. While the objective value of $\AlgBarrier$ clearly exceeds the three other baseline algorithms, its computational complexity follows the same behavior as \cref{fig:video-o-budget}.
Another important observation is that both Greedy and Density Greedy do not have consistent performance across different applications. 
For example, while in the experiments of \cref{fig:graph-f-1912} in \cref{sec:vertex-cover} the Greedy algorithm returns solutions with much higher utilities than Density Greedy, as we see in \cref{fig:video-f-budget}, the performance of Density Greedy is even slightly better than Fast for the video summarization task.
It is worthwhile to mention that, by increasing the value of $m_i$, the maximum cardinality of a feasible solution $r$ increases linearly; as stated by Theorem~\ref{thm:matroids+knapsacks}, the computational complexity of \AlgBarrier increases (see \cref{fig:video-o-g}).

\begin{figure*}[htb!] 
	\centering  
	\subfloat[$m_i = 10$] {\includegraphics[height=32.7mm]{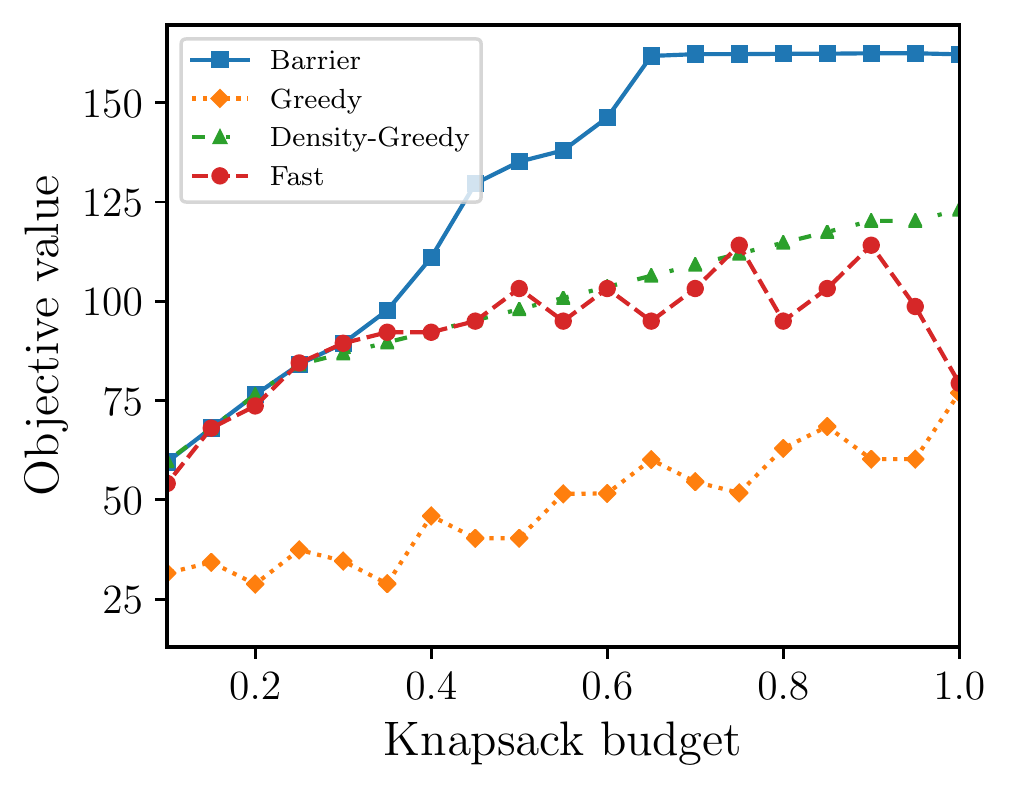}\label{fig:video-f-budget}}
	\subfloat[Knapsack budget $= 1.0$]	
	{\includegraphics[height=32.7mm]{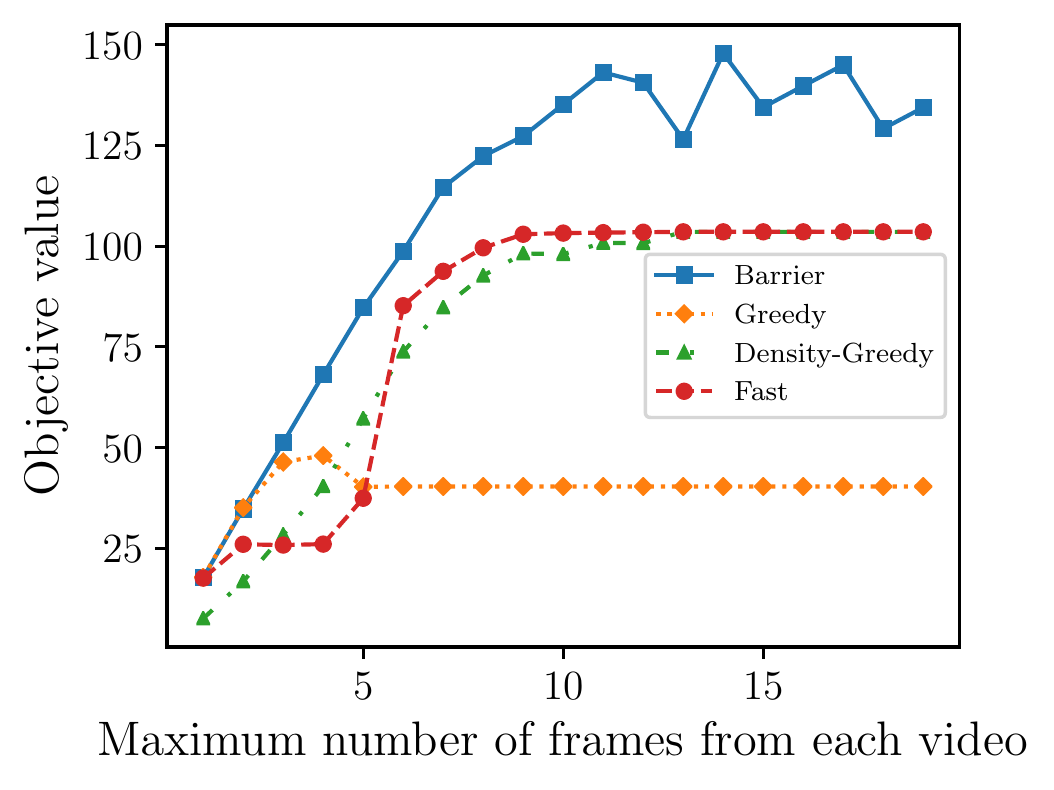}\label{fig:video-f-g}}
	\subfloat[$m_i = 10$]
	{\includegraphics[height=32.7mm]{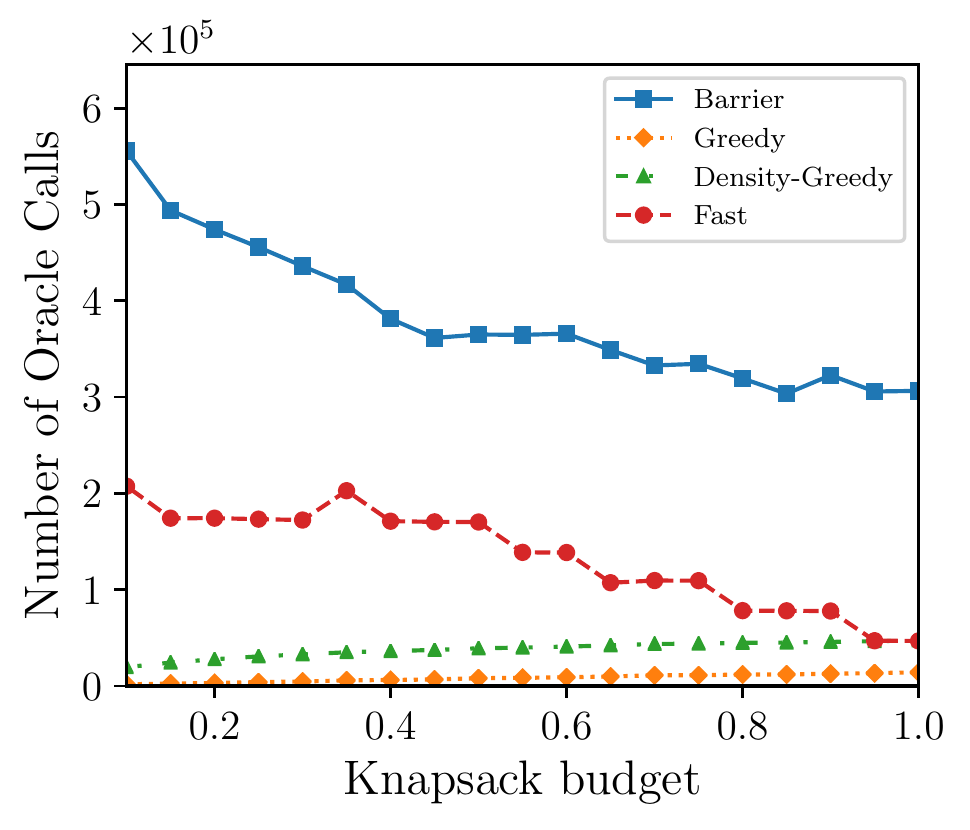}\label{fig:video-o-budget}}
	\subfloat[Knapsack budget $= 1.0$]	{\includegraphics[height=32.7mm]{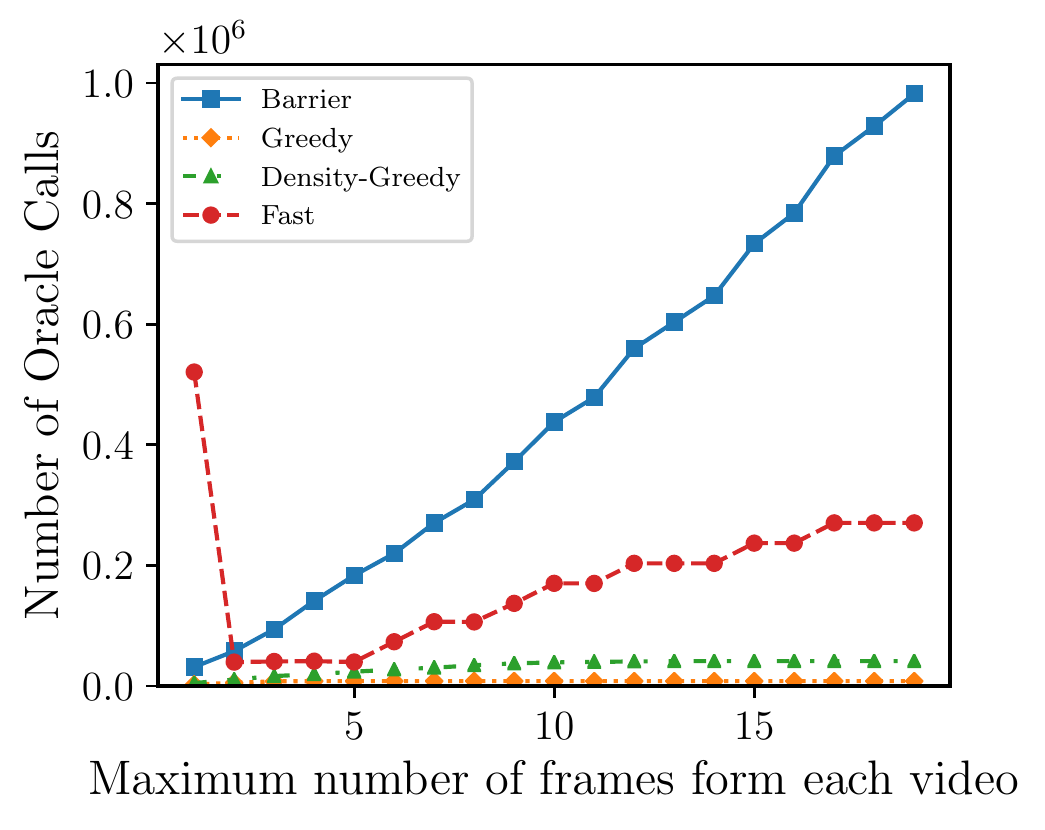}\label{fig:video-o-g}}
	\label{fig:video}
	\caption{We summarize a collection of five different videos. (a) and (c) compare algorithms for varying knapsack budgets. (b) and (d) compare algorithms by changing the limit for the maximum number of allowed frames from each video. We also set $\lambda = 1.0$.}
	\vspace{-15pt}
\end{figure*}

\begin{figure*}[ht] 
	\centering  
	\subfloat[$m=30, m_i=10, \lambda=1.0$]	{\includegraphics[height=33.2mm]{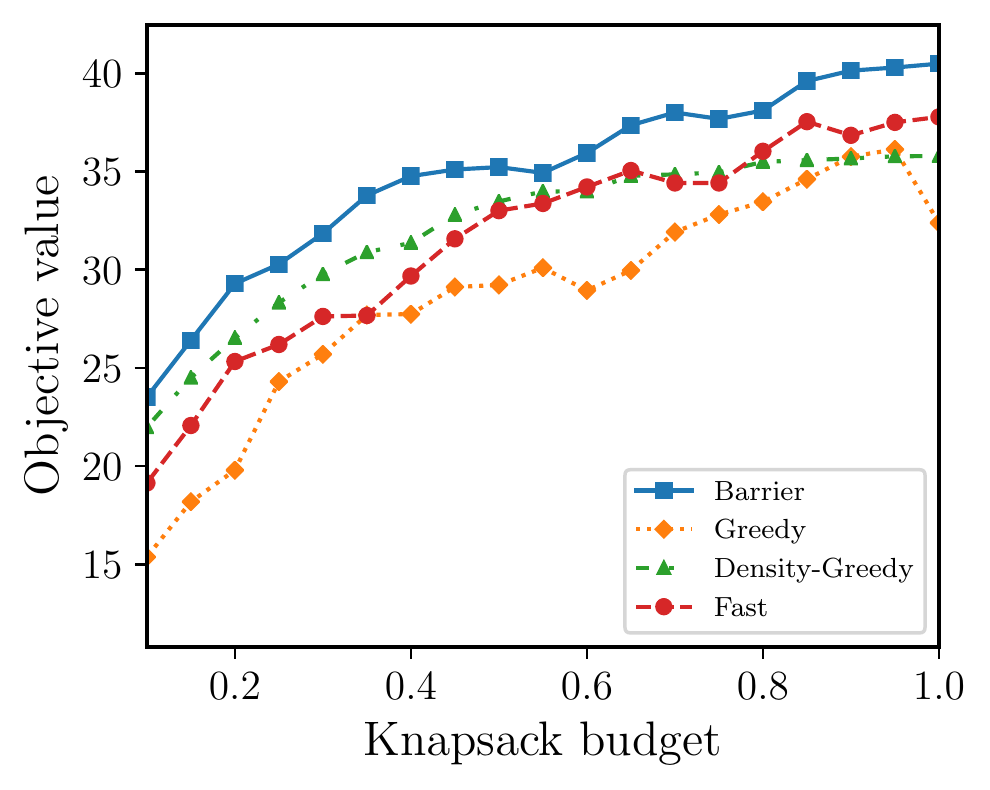}\label{fig:yelp-f-budget}}
	\subfloat[$\textrm{B}=1, m_i=20, \lambda=0.1$]
	{\includegraphics[height=33.2mm]{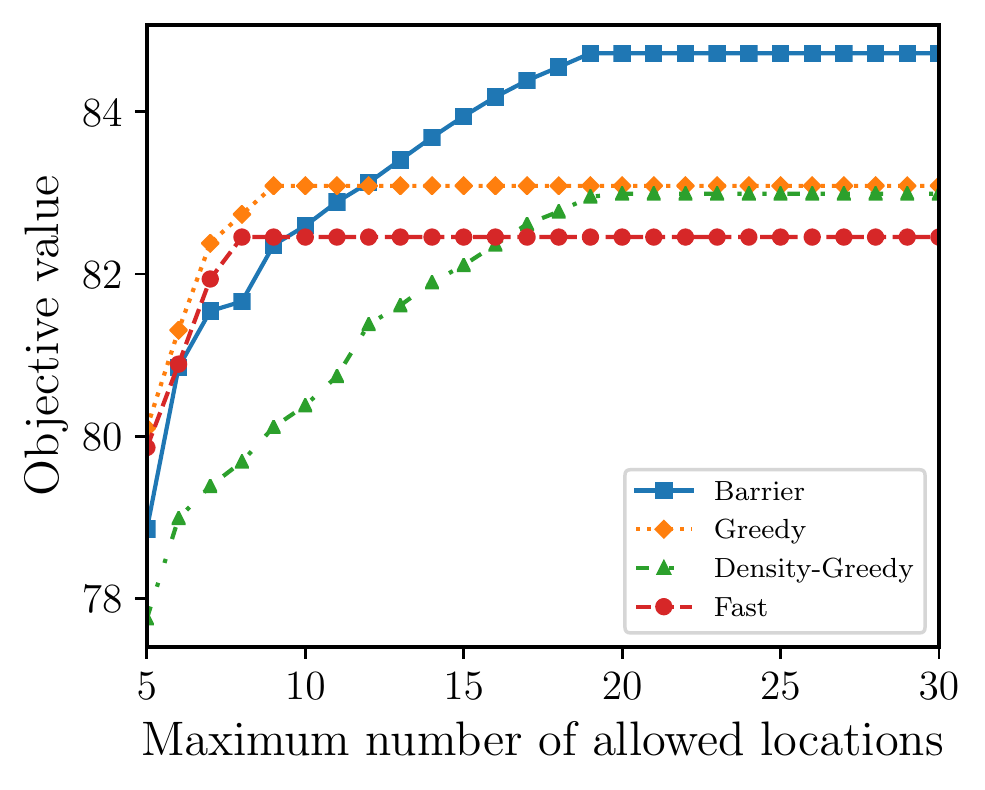}\label{fig:yelp-f-k}}
		\subfloat[$m=30, m_i=10, \lambda=1.0$] {\includegraphics[height=33.2mm]{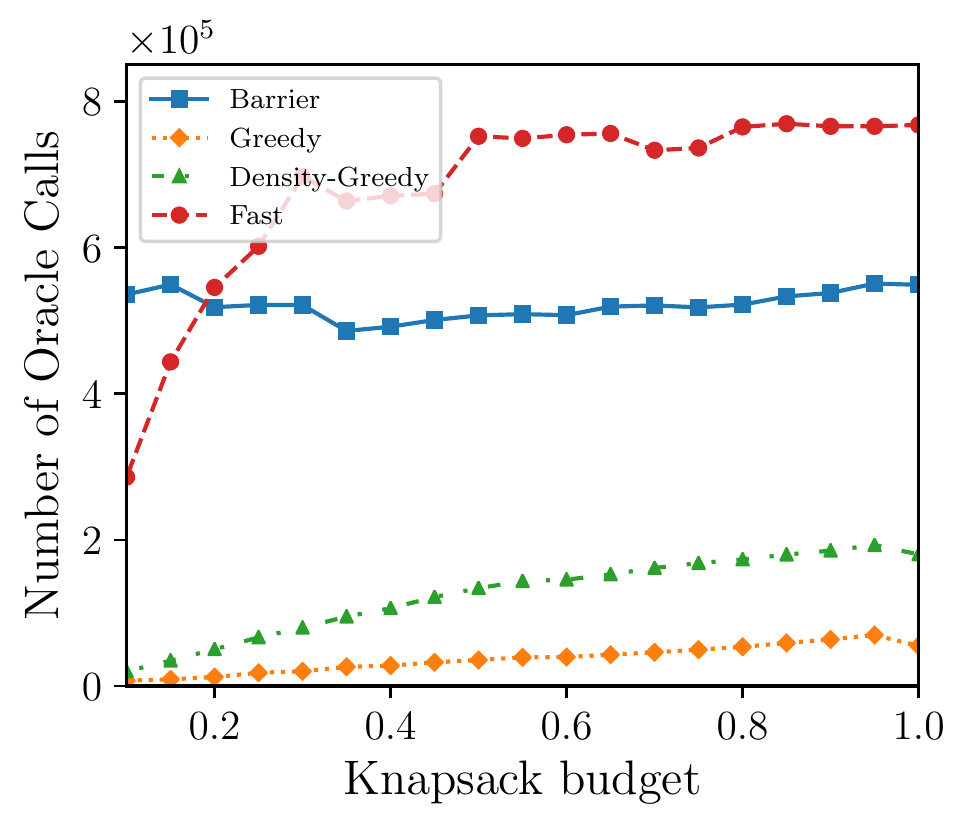}\label{fig:yelp-o-budget}}
	\subfloat[$\textrm{B}=1, m_i=20, \lambda=0.1$]	{\includegraphics[height=33.2mm]{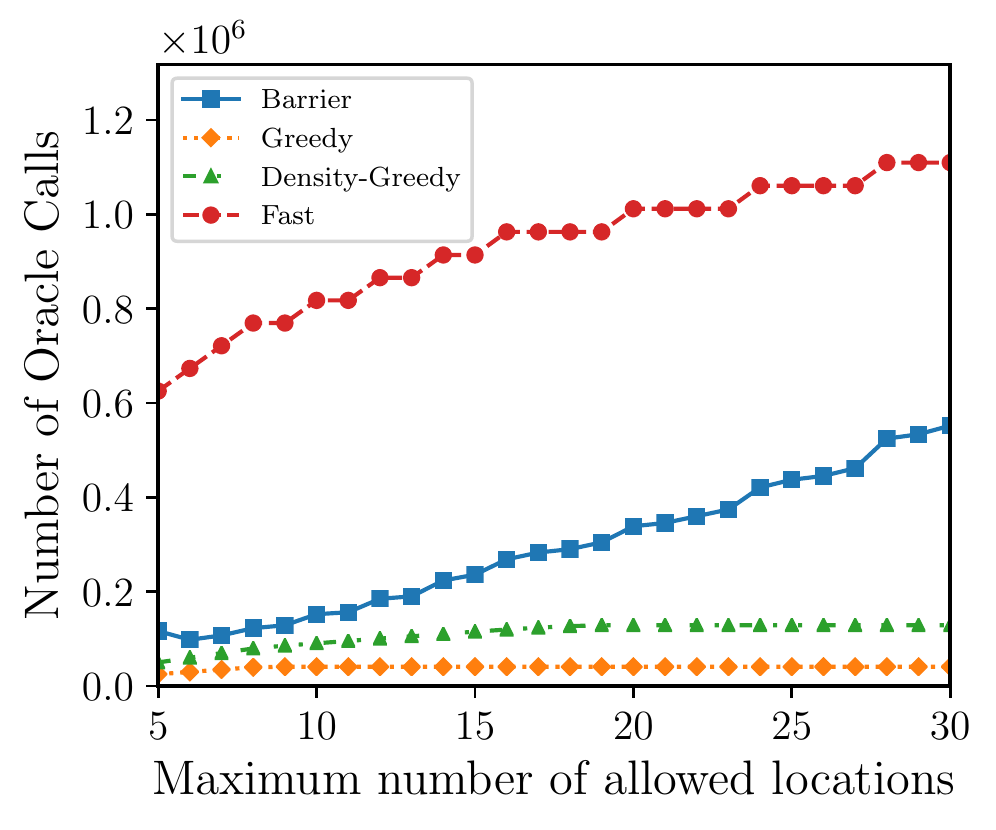}\label{fig:yelp-o-k}}
	\label{fig-yelp}
	\caption{Yelp location summarization: A feasible solution satisfies seven different uniform matroids and three knapsack constraints.}
\end{figure*}

\subsection{More than One Knapsack} \label{sec:multiple-knapsack}

In the first set of experiments, we investigated scenarios where there is only one knapsack constraint. Recall that in \cref{sec:heuristic}, inspired by the main theoretical results of \cref{sec:alg-barrier}, we developed a heuristic algorithm called \AlgHeuristic with the goal of improving the practical performance for cases with multiple knapsacks.
In this section, we report the result of this heuristic algorithm.

 \subsubsection{Yelp Location Summarization}
\label{sec:yelp}

In this section, we consider the Yelp location summarization application, where we have access to thousands of business locations with several related attributes. Our objective is to find a representative summary of the locations from the following cities: Charlotte, Edinburgh, Las Vegas, Madison, Phoenix, and  Pittsburgh.
In these experiments, we use the Yelp Academic dataset \cite{yelp} which is a subset of Yelp's reviews, business descriptions and user data \cite{yelporig}.
For feature extraction, we used the description of each business location and reviews. 
The features contain information regarding many attributes including having vegetarian menus, existing delivery options, the possibility of outdoor seating and  being good for groups.\footnote{Script is provided at \url{https://github.com/vc1492a/Yelp-Challenge-Dataset}.}

Suppose we want to select, out of a ground set $\cN = \{1, \dots , n \} $, a subset of locations that provides a good representation of all the existing business locations.
The quality of each subset of locations is evaluated by a facility location function which we explain next.
A facility at location $i$ is a representative of location $j$ with a similarity value $M_{i,j}$, where $M \in \bR^{n \times n}$. 
For calculating the similarities, similar to the  method described in \cref{sec:movie}, we use $M_{ij} = e^{-\lambda \cdot  \text{dist}(v_i,v_j)}$, where $v_i$ and $v_j$ are extracted feature vectors for locations $i$ and $j$. For a selected set $S$, if each location $i \in \cN$ is represented by a location from set $S$ with the highest similarity, the total utility provided by a set $S$ is modeled by the following monotone and submodular set function \cite{krause12survey,frieze1974cost}:
$f(S) =\frac{1}{n} \sum_{i=1}^{n} \max_{j \in S} M_{i,j}.$

For this experiment, we impose a combination of several constraints: (i) there is a limit $m$ on the total size of summary, (ii) the maximum number of locations from each city is $m_i$, and (iii) three knapsacks $c_1, c_2,$ and $c_3$ where $c_i(j) = \textrm{distance}(j, \textrm{POI}_i)$ is the distance of location $j$ to a point of interest in the corresponding city of $j$. For POIs we consider down-town, an international airport and a national museum in each one of the six cities.
One unit of budget is equivalent to 100km, which means the sum of distances of every set of feasible locations to the point of interests (i.e., down-towns, airports or museums) is at most 100km if we set knapsack budget to one.

In \cref{fig:yelp-f-budget,fig:yelp-o-budget}, we evaluate the performance of algorithms for a varying knapsack budget. We set maximum cardinality of a feasible set to $m=30$, the maximum number of allowed locations from each city to $m_i = 10$ and $\lambda$ to $1.0$. These figures demonstrate that \AlgHeuristic has the best performance in terms of objective value and outperforms the Fast algorithm with respect to computational complexity.
In the second set of experiments, in \cref{fig:yelp-f-k,fig:yelp-o-k}, we compare algorithms based on different upper limits on the total number of allowed locations, where we set the knapsack budgets to one, $m_i$ to $20$, and $\lambda$ to  $0.1$.
Again, from our experiments, it is clear that \AlgHeuristic outperforms Fast and the other baseline algorithms by a huge margin in this setting.

\subsubsection{Twitter Text Summarization} \label{sec:twitter}

As of January 2019, six of the top fifty Twitter accounts  are dedicated primarily to news reporting. In this application, we want to produce representative summaries for Twitter feeds of several news agencies with the following Twitter accounts (also known as ``handles''): @CNNBrk, @BBCSport, @WSJ, @BuzzfeedNews, @nytimes, and @espn.
Each of these handles has millions of followers. Naturally, such accounts commonly share the same headlines and it would be very valuable if we could produce a summary of stories that still relays all the important information without repetition. 

In this application, we use the Twitter dataset from \citep{kazemi2019submodular}, where the keywords from each tweet are extracted and weighted proportionally to the number of retweets the post received. In order to capture diversity in a selected set of tweets,  similar to the approach of \citet{kazemi2019submodular}, we define a monotone and submodular function $f$ defined over a ground set $\cN$ of tweets, where we take the square root of the value assigned to each keyword. 
Each tweet $e \in \cN$ consists of a positive value $\text{val}_e$ denoting its number of retweets and a set of $\ell_e$ keywords $\cW_e = \{ w_{e,1}, \cdots, w_{e, \ell_e}\}$ from the set of all existing keywords $\cW$.
For a tweet $e$,  the score of a word $w \in \cW_e $ is defined by $\text{score}(w,e) = \text{val}_e$.  If $w \notin \cW_e $, we  define $\text{score}(w,e) = 0$.
The function $f$, for a set $S \subseteq \cN$ of tweets, is defined as follows:
\[f(S) = \sum_{w \in \cW} \sqrt{\sum_{e \in S} \text{score}(w,e)} \enspace .\]

A feasible summary should have at most five tweets from each one of the accounts with an upper limit of $15$ on the total number of tweets. Again, this constraint is the intersection of a uniform matroid and a partition matroid. In addition, it should satisfy existing knapsack constraints.
For the first knapsack $c_1$, the cost of each tweet $e$ is weighted proportionally to the difference between the time of $e$ and January 1, 2019, i.e., $c_1(e) \propto |\printtime - \textrm{T}(e)|$. The goal of this knapsack is to provide a summary that mainly captures the events happened around the beginning of the year 2019.
For the second knapsack $c_2$ the cost of tweet $e$ is proportional to the length of each tweet $|\cW_e|$  which enables us to provide shorter summaries.
Each unit of knapsack budget is equivalent to roughly 10 months for $c_1$ and 26 keywords for $c_2$, respectively.

In \cref{fig:twitter-f-1,fig:twitter-o-1}, we compare algorithms under only one knapsack constraint. Similar to the trends in the previous experiments, we observe that \AlgBarrier provides the best utilities, where its number of Oracle calls is competitive with respect to Fast.
In \cref{fig:twitter-f-2,fig:twitter-o-2}, we report the experimental results subject to two knapsacks $c_1$ and $c_2$.
We see that \AlgHeuristic returns the solutions with the highest objective values with a fewer number of calls to the Oracle with respect to Fast. 
We should emphasize that both Greedy and Density Greedy algorithms, due to their simplicity and lack of theoretical guarantees, have the lowest computational complexities.
Finally, by comparing the scenarios with one and two knapsacks, it is evident that having more knapsacks reduces objective values and computational complexity. The main reason for this phenomenon is that by imposing more constraints the size of all feasible sets decreases.

\begin{figure*}[ht] 
	\centering  
	\subfloat[One knapsack $c_1$] {\includegraphics[height=33.mm]{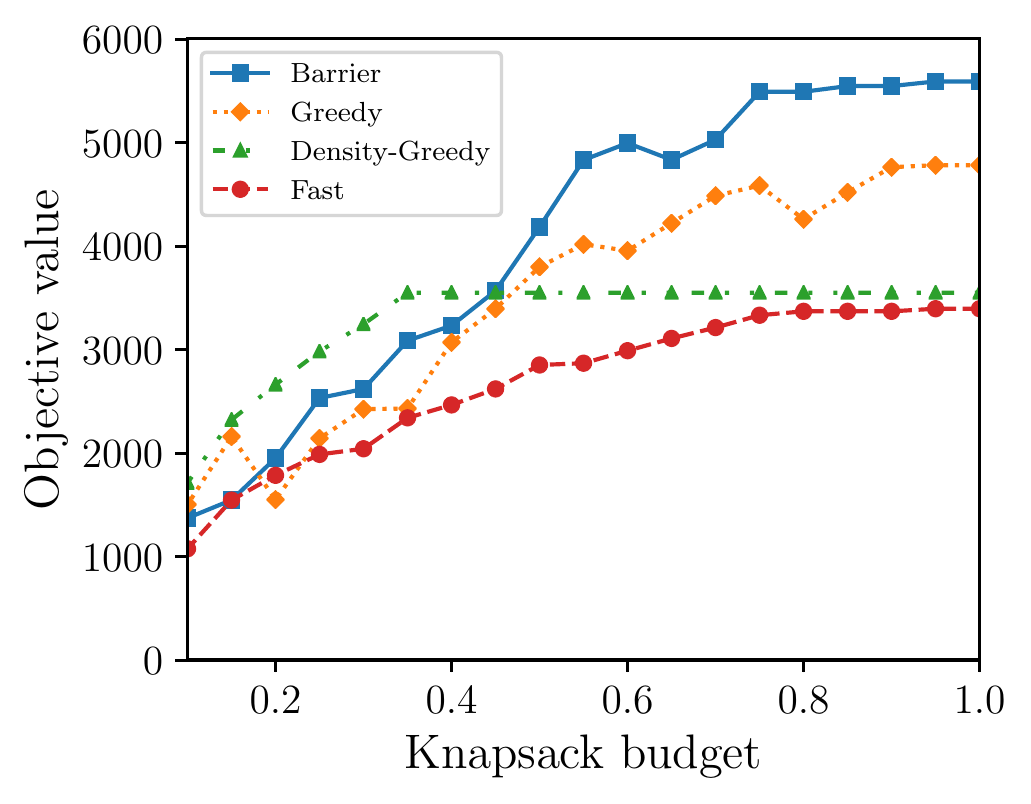}\label{fig:twitter-f-1}}
	\subfloat[Two knapsacks $c_1$ and $c_2$]
	{\includegraphics[height=33.mm]{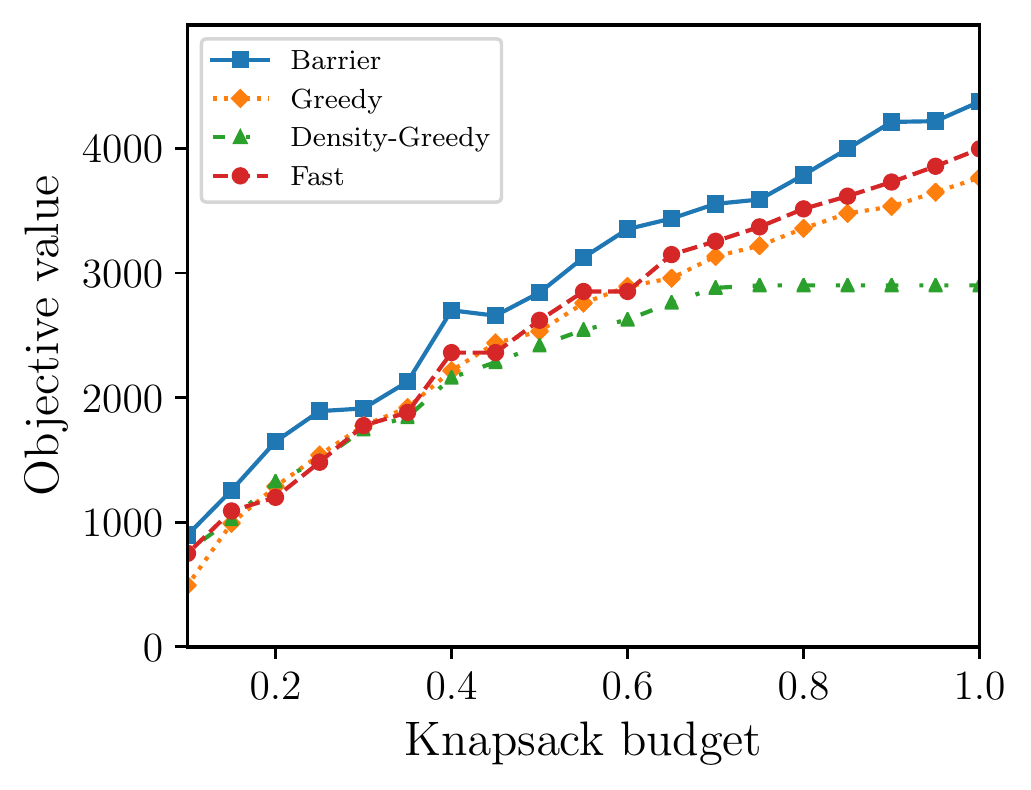}\label{fig:twitter-f-2}}
	\subfloat[One knapsack $c_1$]	{\includegraphics[height=33.mm]{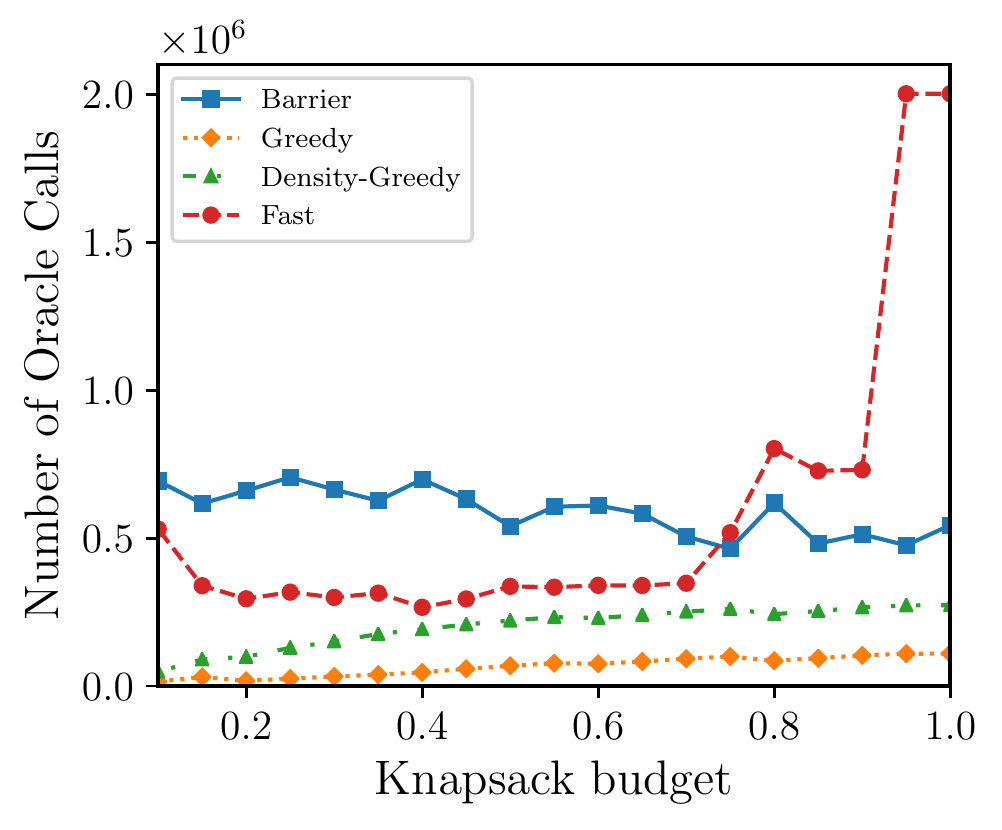}\label{fig:twitter-o-1}}
	\subfloat[Two knapsacks $c_1$ and $c_2$]	{\includegraphics[height=33.mm]{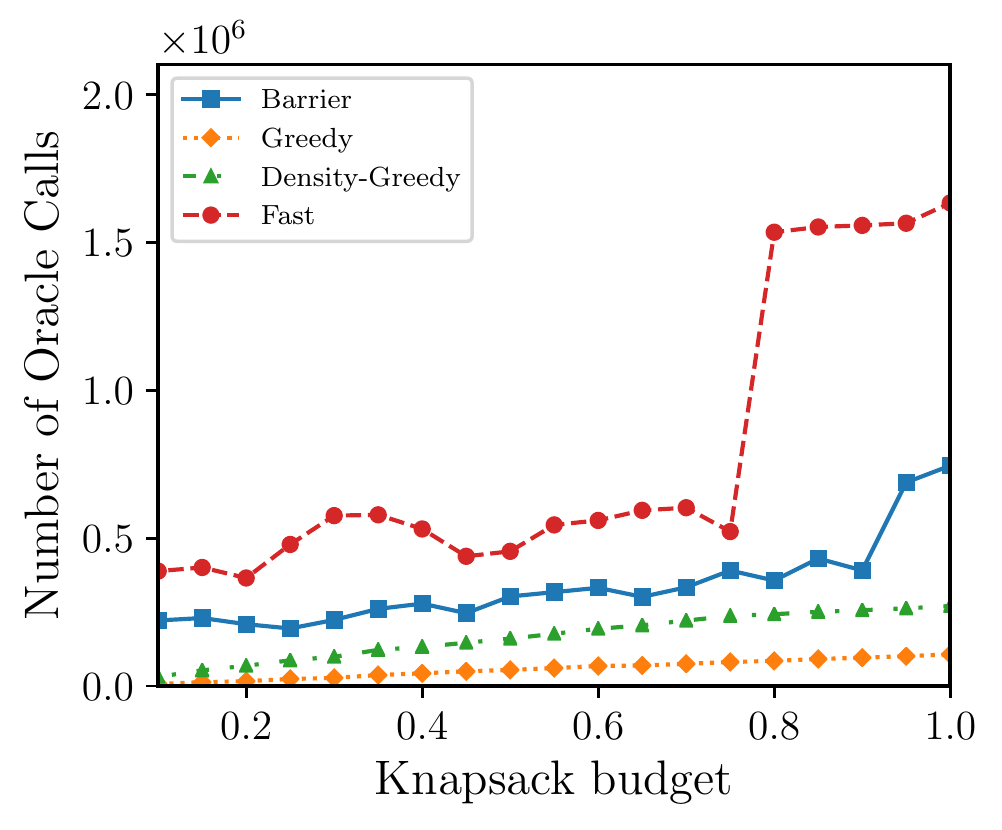}\label{fig:twitter-o-2}}
	\caption{Twitter text summarization: We compare algorithms based on varying knapsack budget. For knapsacks we have $c_1(e) =  |\printtime- \textrm{T}(e)|$ and $c_2(e) = |\cW_e|$.  \label{fig:twitter}}
\end{figure*}

\subsubsection{Movielens Recommendation System} \label{sec:Movielens}

In the final application, our objective is to recommend a set of diverse movies to a user. For designing our recommender system, we use ratings from MovieLens dataset \cite{harper2015movielens}, and apply the method proposed by \citet{lindgren2015sparse} to extract a set of attributes for each movie.
For this experiment, we consider a subset of this dataset which contains 1793 movies from the three genres of Adventure, Animation, and Fantasy. 
For a ground set of movies $\cN$, assume $v_i$ represents the feature vector of the $i$-th movie.
Following the same approach we used in \cref{sec:movie}, we define a similarity matrix $M$ such that $M_{ij} = e^{-\lambda \cdot  \text{dist}(v_i,v_j)}$, where $\text{dist}(v_i,v_j)$ is the euclidean distance between vectors $v_i, v_j \in \cN$. 
The objective of each algorithm is to select a subset of movies that maximizes  the following monotone and submodular function:
$f(S) = \log \det (\mathbf{I} + \alpha M_S)$, where $\mathbf{I}$ is the identity matrix. 

The user specifies an upper limit $m$ on the number of movies for the recommended set, as well as an upper limit $m_i$ on the number of movies from each one of the three genres. This constraint represents a $k$-matchoid independence system with $k = 4$, because a single movie may be identified with multiple genres and the constraint over the genres is not a partition matroid anymore. 
In addition to this $k$-matchoid constraint, we consider three different knapsacks.
For the first knapsack $c_1$, the cost assigned to each movie is proportional to the difference between the maximum possible rating in the iMDB (which is $10$) and the rating of the particular movie---here the goal is to pick movies with higher ratings.
For the second and third knapsacks $c_2$ and $c_3$, the costs of each movie are proportional to the absolute difference between the release year of the movie and the year $1990$ and year $2004$.
The implicit goal of these knapsack constraints is to pick movies with a release year which is as close as possible to these years. 
More formally, for a movie $v \in \cN$, we have: $c_1(v) = 10 - \textrm{rating}_v $, $c_2(v) = \lvert 1990 - \textrm{year}_v \rvert$, and $c_3(v) = \lvert 2004 - \textrm{year}_v \rvert$. Here, $\textrm{rating}_v$ and $\textrm{year}_v$, respectively, denote the IMDb rating and the release year of movie $v$.
We normalize the knapsacks such that the average cost of each movie is $\nicefrac{1}{10}$, i.e., $\frac{\sum_{v \in \cN}c_i(V)}{|\cN|} = \nicefrac{1}{10}$.
For simplicity, we use a single value $m_i = 20$ for all genres, and we set $\lambda = 0.1$. 

In \cref{fig:movie-f-k,fig:movie-o-k}, we evaluate the algorithms for varying the maximum number of allowed movies in the recommendation. For the knapsacks, we consider $c_1$ and $c_2$.
In this experiment, we set the knapsack budget to $\nicefrac{1}{4}$.
In \cref{fig:movie-f-budget,fig:movie-o-budget}, we compare algorithms based on different values of the knapsack budget, where we consider all the three knapsack constraints.
In both of these settings, we again confirm that \AlgHeuristic, with a very modest computational complexity, outperform state-of-the-art algorithms in terms of the quality of recommended movies.

\begin{figure*}[ht] 
	\centering  
	\subfloat[Two knapsacks] {\includegraphics[height=33.mm]{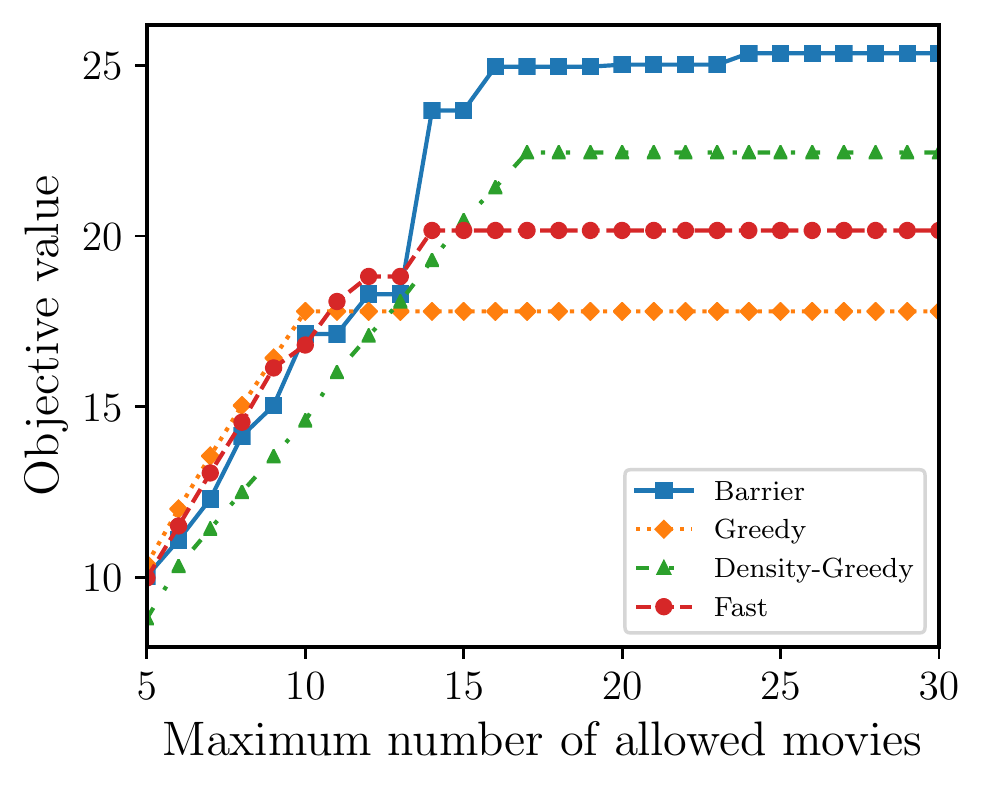}\label{fig:movie-f-k}}
	\subfloat[Three knapsacks]
	{\includegraphics[height=33.mm]{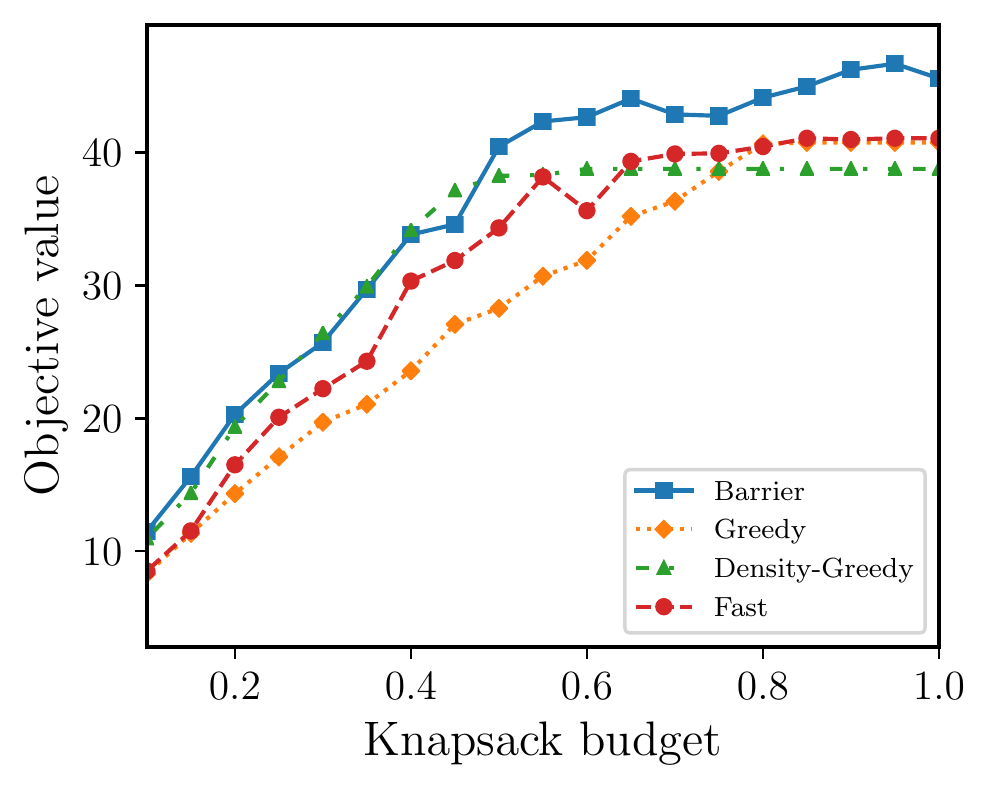}\label{fig:movie-f-budget}}
	\subfloat[Two knapsacks]	{\includegraphics[height=33.2mm]{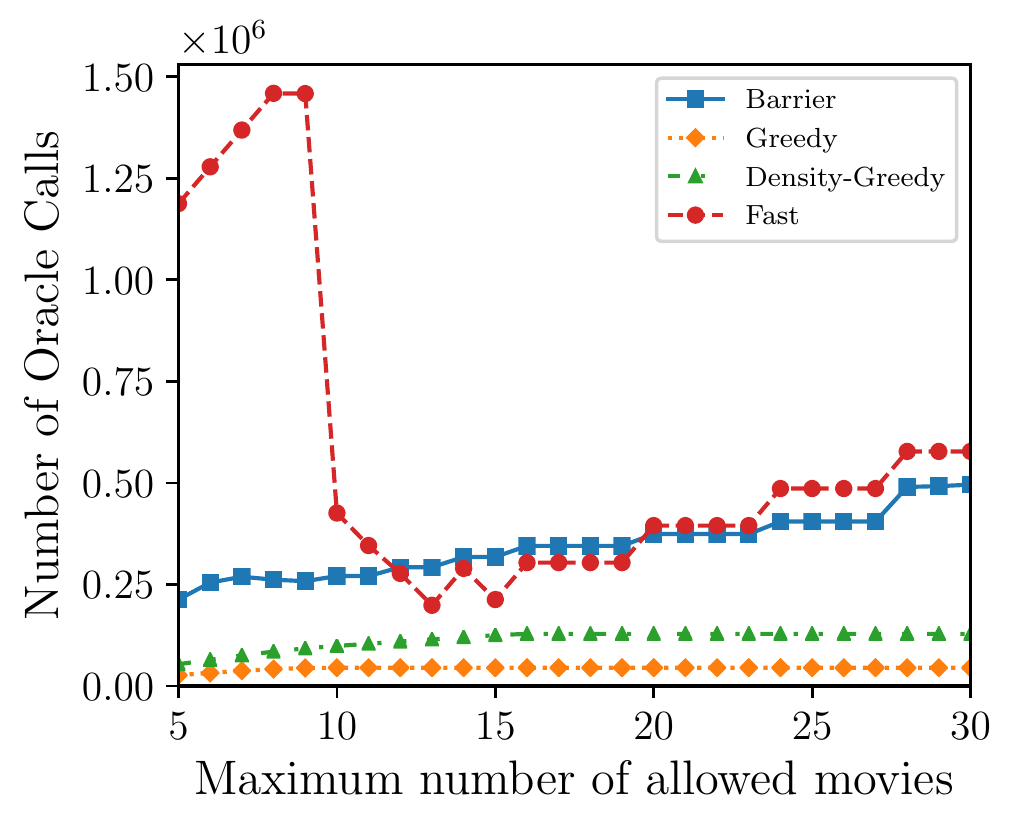}\label{fig:movie-o-k}}
	\subfloat[Three knapsacks]	{\includegraphics[height=33.2mm]{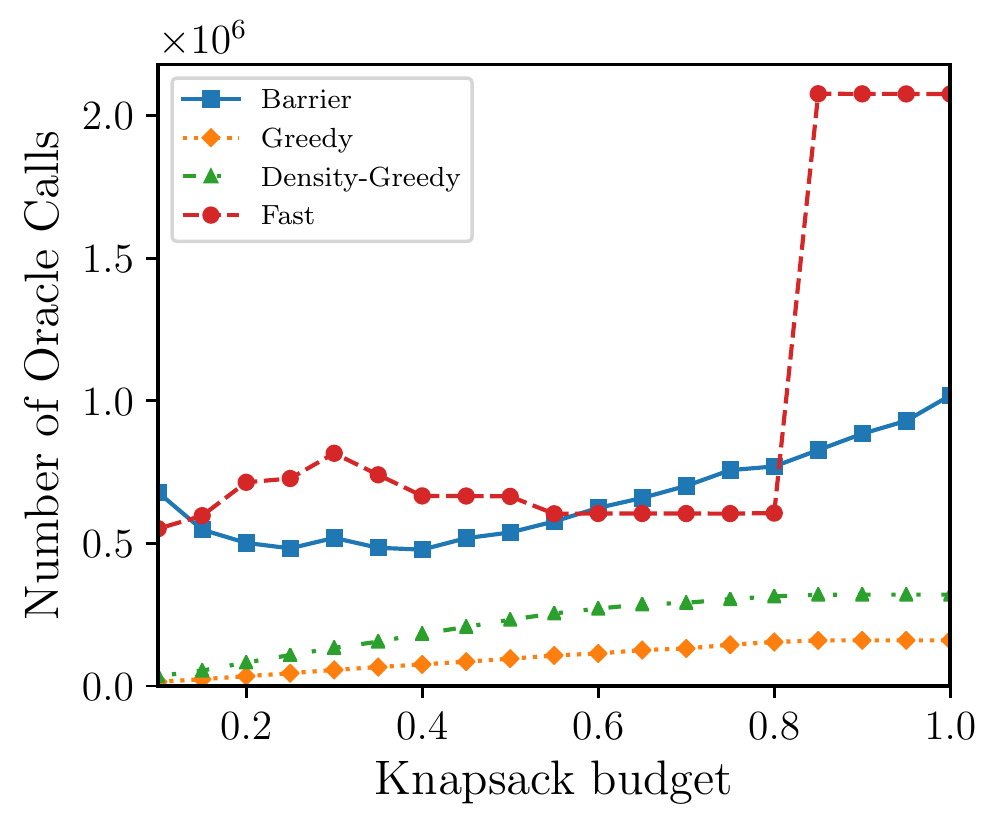}\label{fig:movie-o-budget}}
	\label{fig:movie}
	\caption{Movie recommendation: We compare the performance of algorithms over the Movielens dataset. In (a) and (c), we set the knapsack budget to $\nicefrac{1}{4}$.
		In (b) and (d), we set the maximum cardinality of a feasible solution to $30$.
		We set $\lambda = 0.1$.}
\end{figure*}

\section{Conclusion} \label{sec:conclusion}
	In this paper, we introduced a novel technique for constrained submodular maximization by borrowing the idea of barrier functions from continuous optimization domain. By using this new technique, we proposed two algorithms for maximizing a monotone and submodular function subject to the intersection of a $k$-matchoid and $\ell$ knapsack constraints. The first algorithm, \AlgBarrier, obtains a $2(k+1+\epsilon)$-approximation ratio and runs in $\tilde{O}(n r^2)$ time, where $r$ is the maximum cardinality of a feasible solution. The second algorithm, \AlgImproved, improves the approximation factor to $(k+1+\epsilon)$ by increasing the time complexity to $\tilde{O}(n^3 r^2)$.	We hope that our proposed method devise new algorithmic tools for constrained submodular optimization that could scale to many previously intractable problem instances. We also extensively evaluated the performance of our proposed algorithm over several real-world applications.
\bibliographystyle{plainnat}
\bibliography{./tex/submodfast}

\end{document}